\documentclass[twocolumn,letterpaper]{IEEEtran}
\usepackage[hidelinks,breaklinks]{hyperref}

\usepackage{algorithm}
\usepackage{algorithmic}
\usepackage{colortbl}
\usepackage{newfloat}
\usepackage{listings}

\usepackage[utf8]{inputenc} 
\usepackage[T1]{fontenc}    

\usepackage{url}            
\usepackage{booktabs}       
\usepackage{amsfonts}       
\usepackage{nicefrac}       
\usepackage{microtype}      
\usepackage{xcolor}         

\usepackage{amssymb}

\usepackage{mathtools}
\usepackage{amsthm}
\usepackage{comment}
\usepackage{multirow} 

\usepackage{microtype}
\usepackage{graphicx}
\usepackage{caption,subcaption}
\usepackage{booktabs} 
\usepackage{booktabs}
\RequirePackage{algorithm}
\RequirePackage{algorithmic}
\usepackage{threeparttable}

\newtheorem{theorem}{Theorem}
\newtheorem{lemma}{Lemma}
\theoremstyle{remark}

\newtheorem{Assumption}{Assumption}

\theoremstyle{remark}

\usepackage[capitalize,noabbrev]{cleveref}

\begin{document}

\title{DP-FedPGN: Finding Global Flat Minima for Differentially Private Federated Learning via Penalizing Gradient Norm}

\author{Junkang Liu, Yuxuan Tian, Fanhua Shang, \textit{Senior Member, IEEE},  Hongying Liu, \textit{Senior Member, IEEE}, \\ Yuanyuan Liu, Junchao Zhou, Daorui~Ding
    \thanks{\IEEEcompsocthanksitem Fanhua~Shang, Junkang~Liu, and Daorui~Ding, Junchao Zhou are with the School of Computer Science and Technology, College of Intelligence and Computing, Tianjin University, Tianjin, 300350, China. E-mail: fhshang@tju.edu.cn.\protect
        \IEEEcompsocthanksitem Hongying Liu is with the Medical College, Tianjin University, Tianjin, 300072, China. E-mail: hyliu2009@tju.edu.cn.\protect
        \IEEEcompsocthanksitem  Yuanyuan~Liu are with the School of Artificial Intelligence, Xidian University, China. E-mail: yyliu@xidian.edu.cn.}
}


\maketitle

\begin{abstract}
To prevent inference attacks in Federated Learning (FL) and reduce the leakage of sensitive information, Client-level Differentially Private Federated Learning (CL-DPFL) is widely used. However, current CL-DPFL methods usually result in sharper loss landscapes, which leads to a decrease in model generalization after differential privacy protection. By using Sharpness Aware Minimization (SAM), the current popular federated learning methods are to find a local flat minimum value to alleviate this problem. However, the local flatness may not reflect the global flatness in CL-DPFL. Therefore, to address this issue and seek global flat minima of models, we propose a new CL-DPFL algorithm, DP-FedPGN, in which we introduce a global gradient norm penalty to the local loss to find the global flat minimum. Moreover, by using our global gradient norm penalty, we not only find a flatter global minimum but also reduce the locally updated norm, which means that we further reduce the error of gradient clipping. From a theoretical perspective, we analyze how DP-FedPGN mitigates the performance degradation caused by DP. Meanwhile, the proposed DP-FedPGN algorithm eliminates the impact of data heterogeneity and achieves fast convergence. We also use R\'enyi DP to provide strict privacy guarantees and provide sensitivity analysis for local updates. Finally, we conduct effectiveness tests on both ResNet and Transformer models, and achieve significant improvements in six visual and natural language processing tasks compared to existing state-of-the-art algorithms. 
The code is available at 
\url{https://github.com/junkangLiu0/DP-FedPGN}
\end{abstract}

\begin{IEEEkeywords}
Federated learning, Differential privacy, Flat minima, Federated learning generalization, Data heterogeneity
\end{IEEEkeywords}

\section{Introduction}
\begin{figure}[tb]
	\centering
	\begin{minipage}[b]{0.155\textwidth}
		\centering
		\subcaptionbox{DP-FedAvg}{\includegraphics[width=\textwidth]{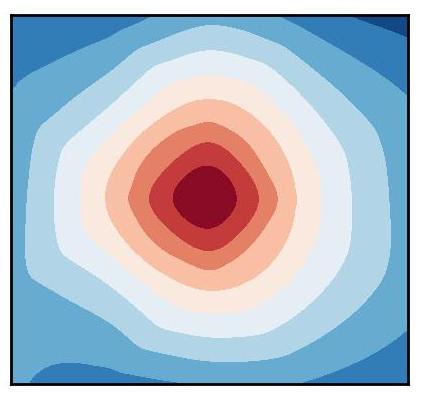}}
	\end{minipage}
	\begin{minipage}[b]{0.155\textwidth}
		\centering
		\subcaptionbox{DP-FedSAM}{\includegraphics[width=\textwidth]{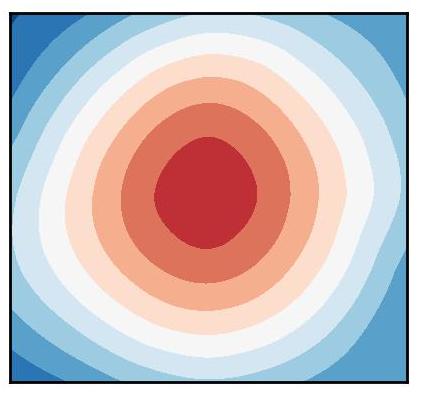}}
	\end{minipage}
	\begin{minipage}[b]{0.155\textwidth}
		\centering
		\subcaptionbox{DP-FedPGN}{\includegraphics[width=\textwidth]{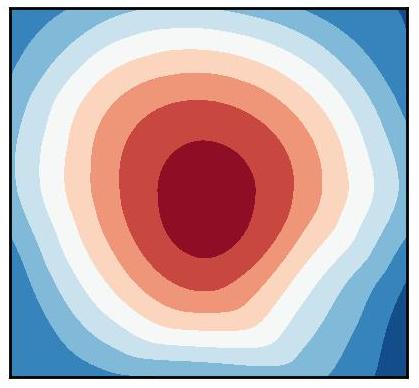}}
	\end{minipage}
\begin{minipage}[b]{0.165\textwidth}
  \centering
  \subcaptionbox{Low heterogeneity}{%
    \raisebox{2mm}{\includegraphics[width=\textwidth]{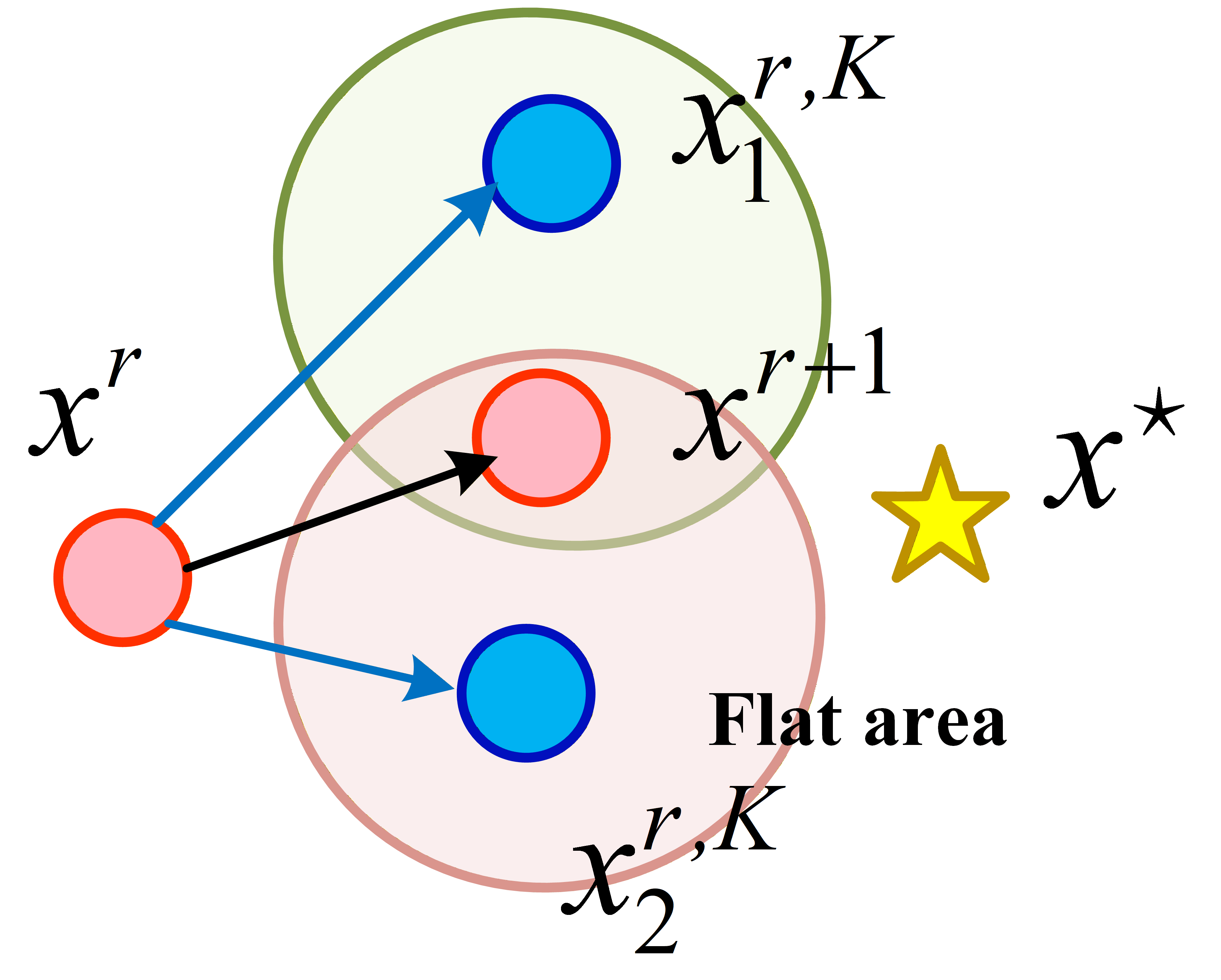}}%
  }
\end{minipage}
	\begin{minipage}[b]{0.15\textwidth}
		\centering
		\subcaptionbox{high heterogeneity}{\includegraphics[width=\textwidth]{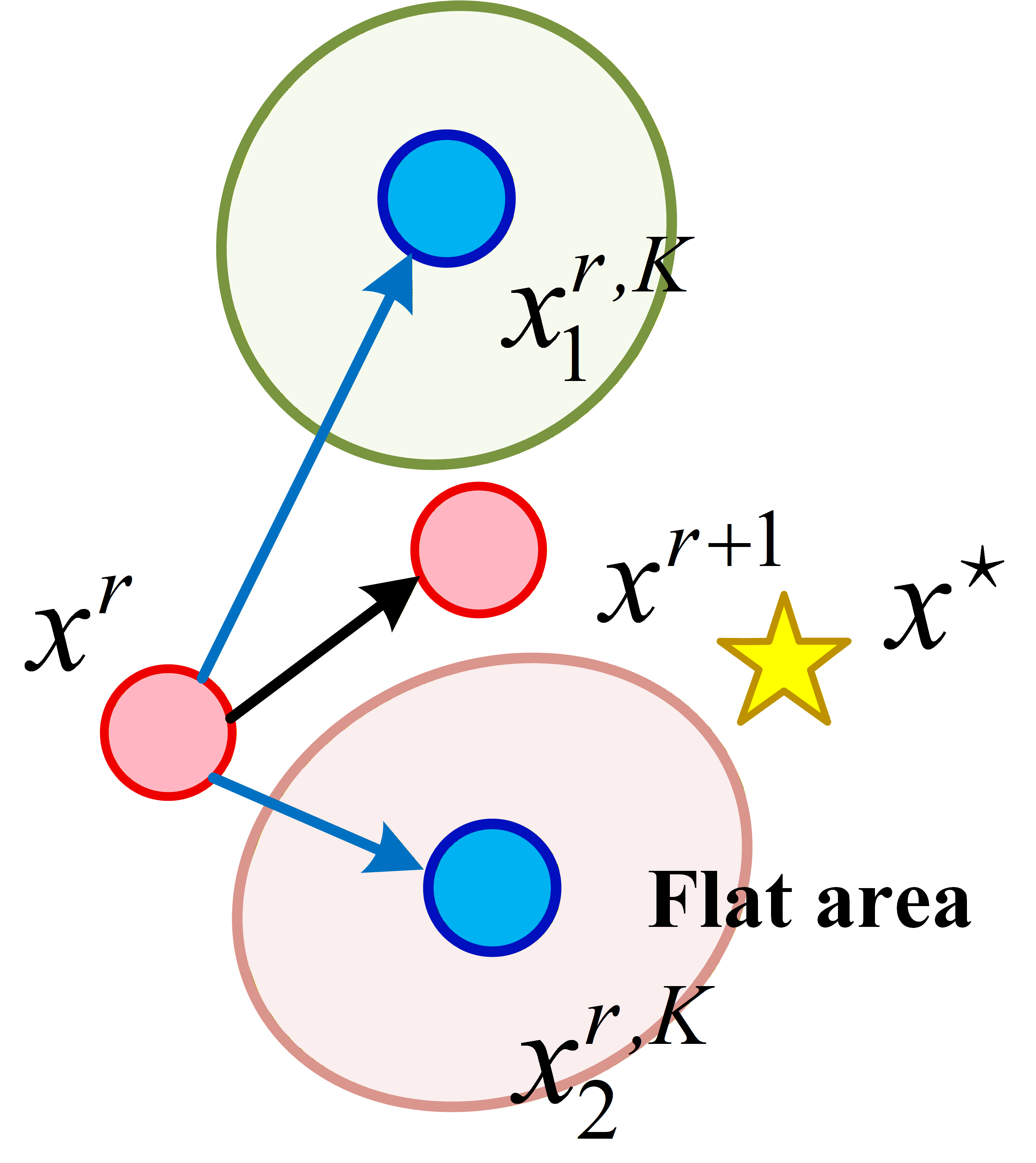}}
	\end{minipage}
    \begin{minipage}[b]{0.16\textwidth}
    \centering
    \subcaptionbox{DP-FedPGN}{%
    \raisebox{3mm}{\includegraphics[width=\textwidth]{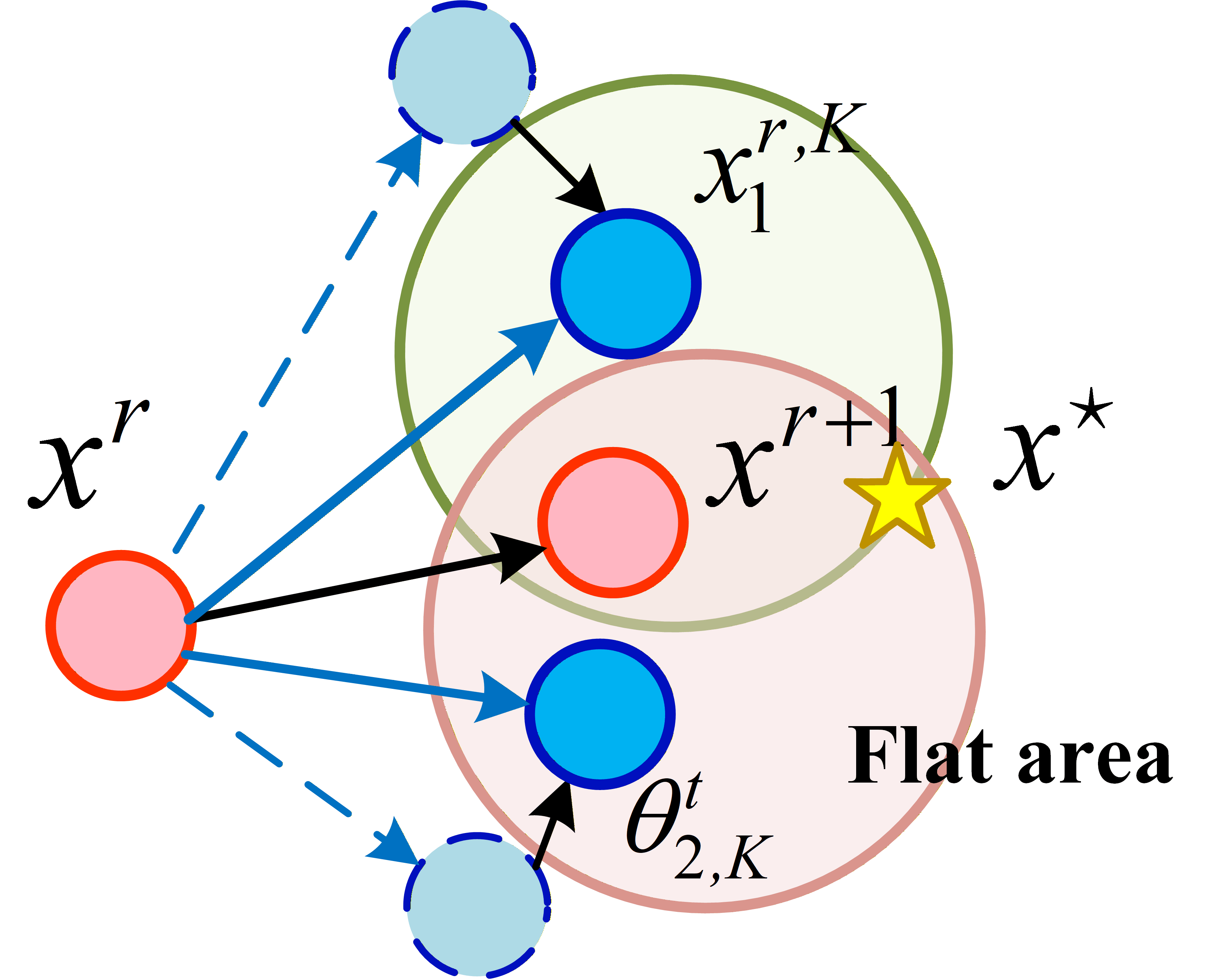}}%
    }
    \end{minipage}
		\caption{The global loss surface for DP-FedAvg \cite{McMahan2018learning}, DP-FedSAM \cite{DP_FedSAM}, and the proposed DP-FedPGN method with ResNet-18 on CIFAR100 in the case of Dirichlet-0.1.}
		\label{figure 1}
\end{figure}
Federated learning (FL) \cite{mcmahan2017communication} enables the training of collaborative models among distributed clients without sharing raw data. Although FL algorithms aim to protect privacy, privacy leakage remains a critical concern \cite{kairouz2021advances,10645291,11219205,LI2026108038,10566979}, as adversaries can infer sensitive client information through model updates. Client-level Differentially Private Federated learning (CL-DPFL) \cite{cheng2022differentially,geyer2017differentially,hu2022federated, KairouzL2021the, liu2024fedbcgd,liuimproving,wu2025llm,wu2025prompt,10.1145/3746027.3755226}, has emerged as a practical standard to safeguard the participation of individual clients, which ensures that adversaries cannot discern whether a client user contributed to training by observing the aggregated model updates\cite{yang2024learning,lin2025semantically,huang2022feature}. Although differential privacy at the client level provides strong privacy guarantees, it can lead to significant performance degradation of existing FL methods.

This phenomenon stems from two critical mechanisms: (1) Gradient clipping bounds the norm of local updates to control sensitivity. This operation inevitably discards valuable information beyond the threshold, particularly harming models trained on non-IID data, where updates exhibit high variability \cite{cheng2022differentially}. (2) The Gaussian noise scaled by the clipping norm introduces variance across clients, which exacerbates model inconsistency in heterogeneous data settings and damages the accuracy of the global model \cite{zhang2022understanding}.
\begin{figure*}[tb]
	\centering
	\begin{minipage}[t]{0.23\textwidth}
		\centering
		\subcaptionbox{DP-FedSAM-Global}{\includegraphics[width=\textwidth]{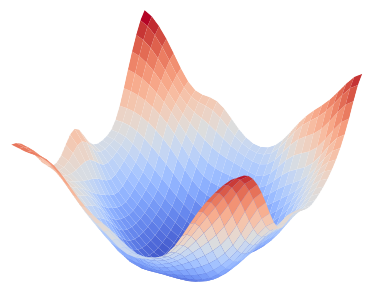}}
	\end{minipage}
	\begin{minipage}[t]{0.23\textwidth}
		\centering
		\subcaptionbox{DP-FedSAM-Local}{\includegraphics[width=\textwidth]{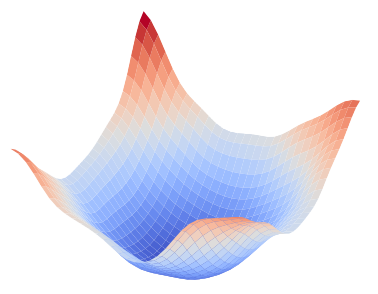}}
	\end{minipage}
	\begin{minipage}[t]{0.23\textwidth}
		\centering
		\subcaptionbox{DP-FedPGN-Global}{\includegraphics[width=\textwidth]{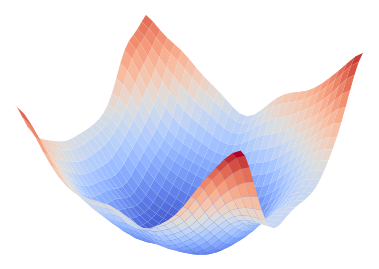}}
	\end{minipage}
	\begin{minipage}[t]{0.23\textwidth}
		\centering
		\subcaptionbox{DP-FedPGN-Local}{\includegraphics[width=\textwidth]{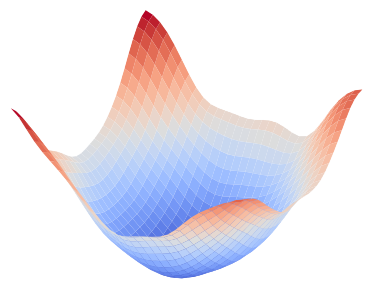}}
	\end{minipage}
		\caption{Loss landscapes of global and local models trained by DP-FedSAM and DP-FedPGN with ResNet-18 on CIFAR100 in the case of Dirichlet-0.1. DP-FedPGN global model is flatter than DP-FedSAM global model.}
		\label{figure 2}
\end{figure*}
A lot of work has been done to address these two issues, and Fed-SMP \cite{Rui2022Federated} mitigates the impact of noise on model updates through sparse updates. DP-FedAvg-BLUR \cite{cheng2022differentially} reduces the threshold for local model variation by adding bounded local update regularization to the loss function. However, these methods have not fundamentally solved the problem of severe performance degradation in CL-DPFL models.

Our extensive experiments demonstrate that the composite effect of gradient clipping and Gaussian noise disrupts the geometry of the loss landscape, trapping the model in sharp minima in Figure \ref{figure 1}, which is the fundamental reason for the severe performance degradation of CL-DPFL models. Numerous studies \cite{foret2020sharpness,izmailov2018averaging} have shown that sharp minima can reduce the model's generalization ability and robustness. The latest method, DP-FedSAM \cite{DP_FedSAM}, attempts to flatten local loss landscapes by locally using  Sharpness-Aware Minimization (SAM) \cite{kwon2021asam}, thereby mitigating the model's reduced generalization caused by DP noise. However, DP-FedSAM mainly focuses on reducing local sharpness without aligning local and global flatness objectives. We find through extensive experiments that in DPFL, the loss landscape flatness of local models may not reflect the loss landscape flatness of the global model after server aggregation in DP-FedSAM, resulting in the aggregated model being located in sharp areas (see Figure \ref{figure 1} (b) and \ref{figure 2} (a, b)).


To address this issue, 
we impose an additional penalty on the gradient norm of the loss function of the global model. The motivation for penalizing the gradient norm of the global loss function is to encourage the optimizer to find the minima located in a relatively flat neighborhood region, as such flat minima have been shown to achieve better model generalization than sharp minima \cite{foret2020sharpness,izmailov2018averaging}. Figure \ref{figure 4} provides a simple example that intuitively illustrates the relationship between gradient norm and minimum flatness.
Previous literature \cite{zhao2022penalizing} has shown that the model is flatter at the minima where the gradient norm is smaller. However, in federated learning, accessing the gradient of the global model is impractical, so we use the global variation of the model to estimate the gradient of the global model. We refer to this algorithm as differentially private federated learning via global penalizing gradient norm (DP-FedPGN). Through extensive experiments, it has been shown that DP-FedPGN can not only find a flatter global minimum than DP-FedSAM and DP-FedAvg
in Figure \ref{figure 1}(b,c), \ref{figure 2}, and \ref{figure 3}, but also lower the threshold for local model variation in Figure \ref{figure 6}.

In addition, Laplacian smoothing (LS) has been widely proven to reduce the noise of gradients, such as in stochastic gradient descent (SGD) \cite{osher2022laplacian}. Laplacian smoothing can also reduce the Gaussian noise introduced by DP \cite{wang2020dp,DP-Fed-LS}. Moreover, our experiment reveals that Laplacian smoothing can also make the model's minima flatter in Figure \ref{figure 8}. Therefore, we add Laplacian smoothing to DP-FedPGN and proposed the DP-FedPGN-LS algorithm.
\begin{figure}[tb]
    \centering
    \includegraphics[scale=0.25]{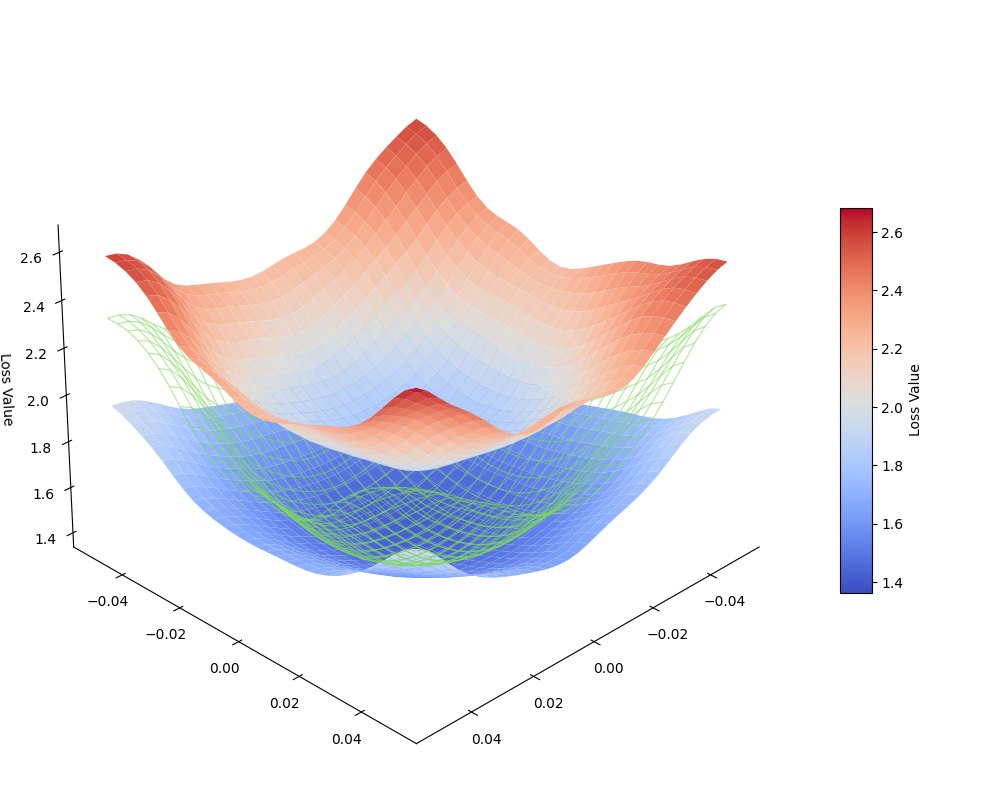}
    \caption{Comparison of loss landscapes of DP-FedAvg (up), DP-FedSAM (middle), and DP-FedPGN (down) on CIFAR10 with ResNet-18 in the case of Dirichlet-0.1.}
    \label{figure 3}
\end{figure}

\subsection{Motivation}
The existing DPFL approaches, including DP-FedAvg \cite{McMahan2018learning}, Fed-SMP \cite{Rui2022Federated}, DP-FedAvg with BLUR and LUS \cite{cheng2022differentially}, and DP-FedSAM \cite{DP_FedSAM}, all suffer from a critical limitation: exacerbated client-side model divergence and convergence to sharp minima in the global model. These issues primarily stem from the dual mechanisms of gradient truncation and stochastic noise injection. While DP-FedSAM attempts to address optimization challenges through local Sharpness-Aware Minimization (SAM), minimizing local sharpness in this way may not be effective in guiding the aggregated model to a global flat minimum in heterogeneous data distributions. Unlike existing works, we try to mitigate these problems by imposing a penalty term to reduce the gradient norm and flatten the global loss landscape. 
Based on the above research, we are the first to introduce the gradient norm penalty into the DPFL framework while considering global flatness so as to effectively mitigate the performance degradation in heterogeneous data settings. 

\subsection{Contributions}
In this paper, we propose a new DPFL algorithm and detailed theoretical analysis to simultaneously improve the generalization and optimization of DPFL, which can be summarized as follows:\\
\textbf{$\bullet$ } We find through experiments that in differentially private federated learning, models often fall into sharp minima. Simultaneously there is a pronounced discrepancy between the flatness of the local loss landscape and the global loss landscape. To address this issue, we propose the global penalty on gradient norm (GPGN) to effectively guide models toward global flat minima.\\
\textbf{$\bullet$ } We propose DP-FedPGN, a novel and effective algorithm that minimizes global sharpness by global gradient norm penalty. We also find that Laplacian smoothing can further flatten the loss landscape of the global model. This insight motivates the development of DP-FedPGN-LS. Theoretically, we demonstrate that our algorithms converge faster than DP-FedSAM and they are more robust to data heterogeneity.\\
\textbf{$\bullet$ } Empirically, we consider different data heterogeneity and neural network architectures in DPFL. We conduct comprehensive experiments of our algorithms on six visual and natural language processing benchmark datasets to show superior performance as well as the efficiency and ability to minimize global sharpness.

\section{ Related work}
\textbf{$\bullet$ CL-DPFL.} Prior research in client-level differentially private federated learning (CL-DPFL) has explored diverse strategies to balance privacy and model performance. DP-FedAvg \cite{McMahan2018learning}, as the pioneering framework, enforces privacy by clipping local updates to a fixed norm and injecting Gaussian noise. Fed-SMP \cite{Rui2022Federated} mitigates noise magnitude by sparsifying updates, retaining only the top-k parameters. Methods like Bounded Local Update Regularization (BLUR) and Local Update Sparsification (LUS) \cite{cheng2022differentially} employ bounded local update regularization to implicitly constrain update norms and dynamically prune updates based on magnitude thresholds.
Recently, DP-FedSAM \cite{DP_FedSAM} diverges from these paradigms by integrating Sharpness-Aware Minimization (SAM) into local training, explicitly optimizing for flat minima to enhance noise robustness. However, its local optimizations fail to fully mitigate the accuracy degradation caused by client-level DP noise.\\
\textbf{$\bullet$ Flat Minimization.} The pursuit of flat minima has emerged as a pivotal strategy to enhance generalization in deep learning, tracing back to Hochreiter and Schmidhuber's seminal observation \cite{hochreiter1997flat} that flat minima correspond to better network generalization. Subsequent works have investigated the relationship between flat minima and generalization in deep learning \cite{keskar2017on,neyshabur2017exploring,dinh2017sharp}. While SGD and its momentum variants naturally exhibit an implicit bias towards flat regions through their noise resilience \cite{virmaux2018lipschitz,xie2021a}, modern approaches explicitly optimize for flatness through distinct paradigms. Stochastic Weight Averaging (SWA) \cite{izmailov2018averaging} leverages SGD's inherent exploration by averaging weights from later training stages, effectively constructing solutions in wide basins through cyclical or constant learning rate sampling. Sharpness-Aware Minimization (SAM) \cite{foret2020sharpness} formalizes flatness seeking as a minimax optimization problem, simultaneously minimizing base loss and worst-case perturbation within a local neighborhood. The recently proposed Penalized Gradient Norm (PGN) method \cite{zhao2022penalizing} establishes direct theoretical connections between gradient norms and curvature, explicitly regularizing the Lipschitz constant of loss landscapes. Our work extends PGN to DPFL frameworks, achieving global flatness through gradient norm penalty.

\section{Federated Learning Settings}
\subsection{FL Problem Setup}
FL aims to optimize model parameters with  local clients, i.e., minimizing the following population risk:
\begin{equation}
F(\boldsymbol{\boldsymbol{x}})=\frac{1}{N} \sum_{i=1}^N\left(F_i(\boldsymbol{\boldsymbol{x}}):=\mathbb{E}_{\zeta_i\sim \mathcal{D}_i}\left[F_i\left(\boldsymbol{\boldsymbol{x}} ; \zeta_i\right)\right]\right).
	\label{eq 1}
\end{equation}
$F_i$ represents the loss function on client $i$. $\mathbb{E}_{\zeta_i \sim \mathcal{D}_i}[\cdot]$ denotes the conditional expectation with respect to the sample $\zeta_i$. $\zeta_i$ is drawn from  distribution $\mathcal{D}_i$ in client $i$.  $N$ is the number of clients.

\begin{figure}[tb]
    \centering
    \includegraphics[scale=0.6]{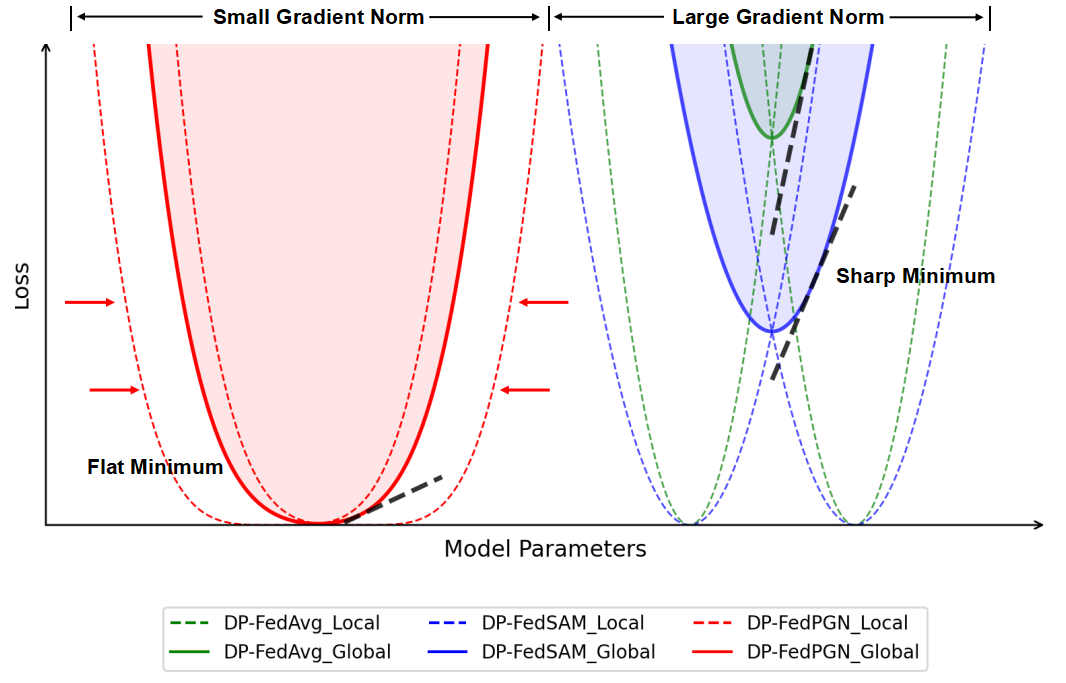}
    \caption{This example illustrates the relationship between the gradient norm of loss and the flatness of its landscape. The smaller the gradient norm, the flatter the loss landscape. Our DP-FedPGN has smaller gradient norm than DP-FedSAM and DP-FedAvg. Moreover, the flat minima of DP-FedPGN is closer than others under data heterogeneity settings.}
    \label{figure 4}
\end{figure}
\subsection{Differential Privacy}
Differential Privacy (DP) \cite{dwork2014algorithmic} can mitigate privacy
risks in federated learning by introducing noise into machine learning models. \\
\textbf{Definition 1 (Rényi Differential Privacy \cite{mironov2017renyi}).} \textit{Let $\alpha > 1$ and $\epsilon \geq 0$ be privacy parameters. A randomized mechanism $\mathcal{M}$ satisfies $\epsilon$-Rényi differential privacy of order $\alpha$, denoted as ($\alpha, \epsilon$)-RDP, if for all pairs of adjacent datasets $ D, D' \in \mathcal{D} $, the Rényi divergence satisfies:}
\begin{align}
	D_{\alpha} \left[ \mathcal{M}(D) \| \mathcal{M}(D') \right] := \frac{1}{\alpha - 1} \log \mathbb{E} \left[ \left( \frac{\mathcal{M}(D)}{\mathcal{M}(D')} \right)^{\alpha} \right] \leq \epsilon.
\end{align}

RDP can be used as a privacy definition, compactly and accurately representing guarantees on the tails of the privacy loss. Compared to ($\epsilon, \delta$)-differential privacy, Rényi differential privacy is a strictly stronger definition of privacy. The privacy quantification for each participant is conducted rigorously through the RDP framework, which provides tight bounds on individual privacy loss \cite{he2021tighter}.\\
\textbf{Definition 2 (Client-level DP \cite{McMahan2018learning}).}
\textit{A randomized algorithm $\mathcal{M}$ satisfies $(\epsilon,\delta)$-differential privacy if for any two adjacent datasets $ D, D' \in \mathcal{D} $ and for any subset of outputs $O \subseteq \mathcal{R}$ it holds that}
\begin{align}
	\Pr[\mathcal{M}(D) \in O] \leq e^{\epsilon} \Pr[\mathcal{M}(D') \in O] + \delta,
\end{align}
\textit{where the adjacency relation $U, U'$ is constructed by adding or removing records of any client.}\\
\textbf{Definition 3 ($l_{2}$ Sensitivity \cite{dwork2014algorithmic}).} \textit{Let $F$ be a function, the $L_2$-sensitivity of $F$ is defined as:}
\begin{align}
	S = \max_{\substack{D \simeq D'}} \lVert F(D) - F(D') \rVert_2,
\end{align}
\textit{where the maximization is taken over all pairs of adjacent datasets.}


\section{Proposed Algorithm: DP-FedPGN}

\definecolor{LightRed}{RGB}{255,182,193} 
\definecolor{LightBlue}{RGB}{173, 216, 230}

\begin{algorithm}[tb]
	\caption{\fcolorbox{LightBlue}{LightBlue}{DP-FedPGN}, \fcolorbox{LightRed}{LightRed}{DP-FedPGN-LS} Algorithm.}
    \label{algorithm 1}
	\begin{algorithmic}[1]
		\STATE {initial model $x^0$ and gradient estimate $g^0$, local learning rate $\eta$, global learning rate $\gamma$, momentum $\beta$}
		\FOR{$r = 0, \dots, R-1$}
		\FOR{each client $i \in \{1, \dots, S\}$ sampled from [$N$] in parallel}
		\STATE $x_i^{r,0} \gets \tilde{x}^{r}$
		\FOR{$k = 0, \dots, K-1$}
		\STATE $\delta_{i}^{r,k}=\rho \frac{\tilde{\boldsymbol{g}}^r}{\|\tilde{\boldsymbol{g}}^r\|}$ 
		\STATE $x_i^{r,k+1/2} \gets x_i^{r,k} + \delta_{i}^{r,k}$
		
		\STATE $g_i^{r,k} \gets \beta \nabla F_i(x_i^{r,k}+\delta_{i}^{r,k}; \xi_i^{r,k}) + (1 - \beta)\tilde{g}^r$
		\STATE $x_i^{r,k+1} \gets x_i^{r,k} - \eta g_i^{r,k}$
		\ENDFOR
		\STATE $\bar{\Delta}_i^t=x_i^{r, K}-\tilde{x}^{r}+(1-\beta) K \eta \tilde{g}^r$
		\STATE$\hat{\Delta}_i^t=\bar{\Delta}_i^t \cdot \min \left(1, \frac{C}{\left\|\bar{\Delta}_i^r\right\|_2}\right)$  $\triangleright$ Clipping \STATE$\tilde{\Delta}_i^t=\hat{\Delta}_i^t+\mathcal{N}\left(0, \sigma^2 C^2 \cdot \mathbf{I}_d /S\right) $ $\triangleright$ Adding Noise
		\STATE$ \Delta_i^t=\tilde{\Delta}_i^t-(1-\beta) K \eta \tilde{g}^r$
		\ENDFOR
		\STATE {\fcolorbox{LightBlue}{LightBlue}{$\tilde{g}^{r+1} \gets \frac{1}{\eta SK} \sum_{i=1}^S -\Delta_i^t$}} 
        \STATE {\fcolorbox{LightRed}{LightRed}{$\tilde{g}^{r+1} \gets A_{\sigma_{ls}}^{-1}(\frac{1}{\eta SK} \sum_{i=1}^S -\Delta_i^t$)}} $\triangleright$ Laplacian Smoothing
		\STATE $\tilde{x}^{r+1}\gets \tilde{x}^{r} - \gamma \tilde{g}^{r+1}$
		\ENDFOR
	\end{algorithmic}
\end{algorithm}

\subsection{Discussion}
To mitigate the performance degradation in DPFL due to sharp loss landscapes caused by differential privacy mechanisms, we propose augmenting the global objective with a gradient norm penalty term. This serves two key purposes: 1) Mitigating clipping-induced bias through gradient norm reduction; 2) Encouraging optimizers to converge to flat minima by reducing the global Lipschitz constant, which corresponds to smoother loss landscapes and better generalization.
By applying gradient norm penalty, the global objective $F(\boldsymbol{x})=\frac{1}{N} \sum_i F_i(\boldsymbol{x})$ is modified as:
\begin{equation}
	\tilde{F}(\boldsymbol{x})=\frac{1}{N} \sum_i F_i(\boldsymbol{x})+\lambda \cdot\left\|\nabla_{\boldsymbol{x}}\left(
	F(\boldsymbol{x})\right)\right\|.
\end{equation}
Here $\| \cdot \|$ simplistically denotes the $L^2$-norm and $\lambda$ is the penalty coefficient. The gradient norm penalty term $\lambda \cdot\left\|\nabla_{\boldsymbol{x}}F(\boldsymbol{x})\right\|$
explicitly smooths the global loss landscape, counteracting the sharp minima induced by DP noise. 
Then, we compute the gradient of the total loss:
\begin{equation}
\begin{aligned}
	& \nabla_x\tilde{F}(\boldsymbol{x})=\nabla\left[\frac{1}{N} \sum_i F_i(\boldsymbol{x})+\lambda \cdot\left\| \nabla_{\boldsymbol{x}} F(\boldsymbol{x})\right\|\right] \\
	& =\nabla_x F(\boldsymbol{x})+\lambda \cdot \nabla_{\boldsymbol{x}}\left\| F(\boldsymbol{x})\right\|.
\end{aligned}
\end{equation}
For the penalty term, it can be calculated by the chain rule:
\begin{equation}
\begin{aligned}
\nabla_{\boldsymbol{x}}\|F(\boldsymbol{x})\|=\nabla_{\boldsymbol{x}}^2 F(\boldsymbol{x}) \frac{\nabla_{\boldsymbol{x}} F(\boldsymbol{x})}{\left\|\nabla_{\boldsymbol{x}} F(\boldsymbol{x})\right\|}. \\
\end{aligned}
\end{equation}
Here, $\nabla_{\boldsymbol{x}}^2 F(\boldsymbol{x})$ is the Hessian matrix $\mathbf{H}$. Directly computing such a Hessian matrix becomes computationally prohibitive due to the high dimensionality of the weight space, which necessitates the use of appropriate approximation techniques for this calculation.
We employ a first-order Taylor expansion to approximate the Hessian-vector product:
\begin{equation}
\begin{aligned}
	& \mathbf{H} \frac{\nabla_{\boldsymbol{x}} F(\boldsymbol{x})}{\left\|\nabla_{\boldsymbol{x}} F(\boldsymbol{x})\right\|} \approx \frac{\nabla_{\boldsymbol{x}} F\left(\boldsymbol{x}+\rho \frac{\nabla_{\boldsymbol{x}} F(\boldsymbol{x})}{\left\|\nabla_{\boldsymbol{x}} F(\boldsymbol{x})\right\|}\right)-\nabla_{\boldsymbol{x}} F(\boldsymbol{x})}{\rho},
\end{aligned}
\end{equation}
where $ \rho $ is a small scalar.\\
Now, substituting the approximation into the gradient:
\begin{equation}
\begin{aligned}
	& \nabla_{\boldsymbol{x}}\tilde{F}(\boldsymbol{x})=\nabla_{\boldsymbol{x}} F(\boldsymbol{x})+\frac{\lambda}{\rho} \cdot\left(\nabla_{\boldsymbol{x}} F\left(\boldsymbol{x}+\rho \frac{\nabla_{\boldsymbol{x}} F(\boldsymbol{x})}{\left\|\nabla_{\boldsymbol{x}} F(\boldsymbol{x})\right\|}\right)-\nabla_{\boldsymbol{x}} F(\boldsymbol{x})\right) \\
	& =(1-\beta) \nabla_{\boldsymbol{x}} F(\boldsymbol{x})+\beta \nabla_{\boldsymbol{x}} F\left(\boldsymbol{x}+\rho \frac{\nabla_{\boldsymbol{x}} F(\boldsymbol{x})}{\left\|\nabla_{\boldsymbol{x}} F(\boldsymbol{x})\right\|}\right),
\end{aligned}
\end{equation}
where $\beta = \frac{\lambda}{\rho}$. As for the local client, there is:
\begin{equation}
\begin{aligned}
	& \nabla_{\boldsymbol{x}} \tilde{F}_i(\boldsymbol{x})=(1-\beta) \nabla_{\boldsymbol{x}} F_i(\boldsymbol{x})+\beta \nabla_{\boldsymbol{x}} F_i\left(\boldsymbol{x}+\rho \frac{\nabla_x F_i(\boldsymbol{x})}{\left\|\nabla_{\boldsymbol{x}} F_i(\boldsymbol{x})\right\|}\right)
\end{aligned}
\end{equation}
However, the inherent data heterogeneity induces divergence between local and global models, thereby resulting in 
$
\nabla_{\boldsymbol{x}} F_i(\boldsymbol{x}) \neq \nabla_{\boldsymbol{x}} F(\boldsymbol{x}). 
$
To mitigate client drift induced by data heterogeneity, we replace local gradient computation with global model variation to align local and global optimization direction: 
\begin{equation}
\begin{aligned}
	& \nabla_{\boldsymbol{x}} \tilde{F}(\boldsymbol{x})\approx(1-\beta) g+\beta \nabla_{\boldsymbol{x}} F_i\left(\boldsymbol{x}+\rho \frac{g}{\left\|g\right\|}\right),
\end{aligned}
\end{equation}
\begin{equation}
\begin{aligned}
	&g=\frac{1}{N K \eta} \sum_i \Delta_{i}\space,  \space\space \space  g\approx\nabla_{\boldsymbol{x}} F(\boldsymbol{x})\\
\end{aligned}
\end{equation}
where $\Delta_{i}$ denotes the local update of client $i$. 
\begin{table*}[tb]
	\centering
	\caption{Comparison between DP-FedPGN and DP-FedSAM.} 
	\label{table 1}
	\setlength{\tabcolsep}{0.5pt}	\begin{tabular}{llcccccc}
		\midrule[1pt]
		\centering
		& \textbf{Research work}   &  \textbf{Minimizing Target }  &  \textbf{Local Perturbation} &\textbf{Convergence Rate}& \textbf{Assumption} & \textbf{Computation Cost}\\
		\midrule[0.5pt]		
		&DP-FedSAM \cite{DP_FedSAM}& Local sharpness & $\rho \frac{\nabla F_i\left(x_{i}^{r,k}\right)}{\| \nabla F_i(x_{i}^{r,k} )\|}$ & $\mathcal{O}\left(\sqrt{\frac{L \Delta (\sigma_l^2+\sigma_g^2)}{S K R}}+\frac{L \Delta}{R}+\frac{\sigma^2 C^2}{S^2}\right)$&Data Heterogeneity&2$\times$\\
        
        \rowcolor{LightBlue}
		&\textbf{DP-FedPGN} (ours) & Global Sharpness & $\rho \frac{\tilde{g}^{r}}{\|\tilde{g}^{r}\|}$&  $\mathcal{O}\left(\sqrt{\frac{L \Delta \sigma_l^2}{S K R}}+\frac{L \Delta}{R}+\frac{\sigma^2 C^2}{S^2}\right)$&Without data heterogeneity&1$\times$\\
		\midrule[1.5pt]
	\end{tabular}
	\label{table com}
\end{table*}
\subsection{DP-FedPGN Method}
Our proposed method is displayed in Algorithm \ref{algorithm 1}. In each communication round $r \in\left\lbrace 0, 1, \dots, R-1\right\rbrace $, a subset of $S$ clients is chosen from the total $N$ participants with sampling ratio $q$ to amplify the privacy guarantee. At each local iteration $k \in \left\lbrace 0, 1, \dots, K-1\right\rbrace $, each client $i \in \left\lbrace 1, 2, \dots, S\right\rbrace $ updates the local model as:
\begin{equation}
x_i^{r,k+1} = x_i^{r,k} - \eta g_i^{r,k},
\end{equation}
\begin{equation}
g_i^{r,k} = \beta \nabla F_i(x_i^{r,k}+\delta_{i}^{r,k}; \xi_i^{r,k}) + (1 - \beta)g^r,
\delta_{i}^{r,k}=\rho\frac{\boldsymbol{g}^r}{\|\boldsymbol{g}^r\|},
\end{equation}
where $\beta$ is the momentum rate derived by adding a gradient norm penalty term to the loss function and $g^r$ is the $r$-th round global pseudo-gradient, replacing local gradients to ensure our algorithm is oriented towards global flatness. After that, we implement gradient clipping and add Gaussian noise to ensure client-level DP. Note that the momentum term originates from aggregated updates already processed with differential privacy. Therefore, this term is excluded in the current round to avoid amplifying errors from repeated DP operation:
\begin{equation}
	\hat{\Delta}_i^t=x^r-x_i^{r, K}+(1-\beta) K \eta g^r \\
\end{equation}
For the processed local updates $\hat{\Delta}_i^t$, we perform a clipping operation using a threshold $C$ and add Gaussian noise with a noise standard deviation $\sigma$:
\begin{equation}
\hat{\Delta}_i^t=\hat{\Delta}_i^t \cdot \min \left(1, \frac{C}{\left\|\hat{\Delta}_i^t\right\|_2}\right),
\end{equation}
\begin{equation}
\hat{\Delta}_i^t=\hat{\Delta}_i^t+\mathcal{N}\left(0, \sigma^2 C^2 \cdot \mathbf{I}_d / m\right) 
\end{equation}
Then, we reinstate the term and obtain the final local update:
\begin{equation}
	\Delta_i^t=\hat{\Delta}_i^t-(1-\beta) K \eta g^r\\
\end{equation}
Finally, we aggregate the changes of each client and calculate $g^{r+1}$, and $A_{\sigma_{ls}}^{-1}$ is the Laplacian smoothing matrix.

\section{Theoretical Analysis}

We analyze Convergence based on following assumptions: \\
\textbf{Assumption 1} \textit{(Smoothness).  $F_i$ is $L$-smooth for all $i \in$ $[N]$, 
\begin{equation}
\left\|\nabla F_i(\boldsymbol{x}_{1})-\nabla F_i(\boldsymbol{x}_{2})\right\| \leq L\|\boldsymbol{x}_{1}-\boldsymbol{x}_{2}\|
\end{equation}
for all $\boldsymbol{x}_{1}, \boldsymbol{x}_{2}$ in its domain and $i \in[N]$.}\\
\textbf{Assumption 2} \textit{(Bounded variance of data heterogeneity). The global variability of the local gradient of the loss function is bounded by $\sigma_g^2$, for all $i \in[N]$,} 
\begin{equation}
	\left\|\nabla F_i\left(\boldsymbol{x}\right)-\nabla F\left(\boldsymbol{x}\right)\right\|^2 \leq \sigma_g^2
\end{equation}
\textbf{Assumption 3}\textit{ (Bounded variance of stochastic gradient). The stochastic gradient $\nabla F_i\left(\boldsymbol{x}, \xi_i\right)$, computed by the $i$-th client of model parameter $\boldsymbol{x}$ using mini-batch $\xi_i$ is an unbiased estimator $\nabla F_i(\boldsymbol{x})$ with variance bounded by $\sigma^2$, i.e.,
\begin{equation}
	\mathbb{E}_{\xi_i}\left\|\nabla F_i\left(\boldsymbol{x}, \xi_i\right)-\nabla F_i(\boldsymbol{x})\right\|^2 \leq \sigma_l^2
\end{equation}
	$\forall i \in[N]$, where the expectation is over all local datasets.}\\
\textbf{Assumption 4} \textit{There exists a clipping constant $C$ independent of $i, r$, such that $\left\|\bar{\Delta}_i^r\right\| \leq C$.}

\begin{table*}[tb]
	\centering
	\setlength{\tabcolsep}{4pt}
	\caption{Comparison different data heterogeneity of the accuracy (\%) of each algorithm on  CIFAR10 and CIFAR100 in 300 communication rounds with $E=5$ and ResNet-10. Rounds denote the required rounds to achieve target accuracy.}
	\begin{tabular}{llcccccccc}
		\midrule[1.5pt]
		&\multicolumn{4}{c}{CIFAR10}
		&\multicolumn{4}{c}{CIFAR100}\\ 
		\cmidrule(lr){2-5} 
		\cmidrule(lr){6-9} 
		Method & \multicolumn{2}{c}{Dirichlet-0.1} &  \multicolumn{2}{c}{Dirichlet-0.6} &
		\multicolumn{2}{c}{Dirichlet-0.1} &  \multicolumn{2}{c}{Dirichlet-0.6} \\
		\cmidrule(lr){2-3} \cmidrule(lr){4-5} 
		\cmidrule(lr){6-7} \cmidrule(lr){8-9} 
		& Acc.(\%) & Rounds & Acc.(\%)& Rounds
		& Acc.(\%) & Rounds & Acc.(\%)& Rounds\\
		& 300R & 30\%  &300R & 30\%
		& 300R & 10\%  &300R & 10\%\\
		\midrule
        DP-FedAvg \cite{McMahan2018learning} &45.53 $\pm 0.11$  &40 (1$\times$) &56.47 $\pm 0.65$ &19 (1$\times$) &17.82 $\pm 0.38$ &72 (1$\times$) &21.82 $\pm 0.17$ &42 (1$\times$)  \\
        		
    Fed-SMP \cite{Rui2022Federated}  &45.19 $\pm0.53$ &40 (1$\times$) &55.28 $\pm0.24$ &18 (1.1$\times$) &16.91 $\pm0.72$ &84 (0.9$\times$) &21.15 $\pm0.09$ &44 (1$\times$) \\
    		
    DP-FedAvg-BLUR \cite{cheng2022differentially} &45.84 $\pm0.33$ &40 (1$\times$) &56.31 $\pm0.48$ &18 (1.1$\times$) &16.95 $\pm0.15$ &63 (1.1$\times$) &21.83 $\pm0.61$ &44 (1$\times$)  \\
    		
    DP-FedAvg-LS \cite{DP-Fed-LS}   &46.61 $\pm0.19$   &49 (0.8$\times$) &56.20 $\pm0.55$ &16 (1.2$\times$) &20.63 $\pm0.27$ &61 (1.2$\times$) &24.23 $\pm0.42$ &35 (1.2$\times$) \\
    		
    DP-FedSAM \cite{DP_FedSAM} &38.41 $\pm0.68$  &118 (0.3$\times$)  &45.90 $\pm0.03$ &30 (0.6$\times$)  &13.32 $\pm0.77$ &136 (0.5$\times$) &16.44 $\pm0.29$ &75 (0.6$\times$)  \\
    
    \rowcolor{LightBlue}
    \textbf{DP-FedPGN (ours)}   &58.46 $\pm0.35$ &\textbf{12} (3.3$\times$)  & \textbf{70.98 $\pm0.55$}  &\textbf{4} (4.8$\times$) &\textbf{23.78 $\pm0.18$}  &34 (2.1$\times$) &28.85 $\pm0.62$ &21 (2$\times$)  \\
    \rowcolor{LightRed}
    \textbf{DP-FedPGN-LS (ours)}   &\textbf{61.80 $\pm0.27$}  & \textbf{12} (3.3$\times$)   & 70.95 $\pm0.43$ &\textbf{4} (4.8$\times$) &23.42 $\pm0.76$ &\textbf{21} (3.4$\times$) &\textbf{29.58 $\pm0.09$} &\textbf{19} (2.2$\times$)  \\
    \midrule[1.5pt]
	\end{tabular}
	\label{table 1}
\end{table*}	
\begin{table*}[tb]
	\centering
	\setlength{\tabcolsep}{4pt}
	\caption{Comparison accuracy (\%)  on  CIFAR10 and CIFAR100 in 300 communication rounds with $E=5$ and ResNet-18.}
	\begin{tabular}{llcccccccc}
		\midrule[1.5pt]
		&\multicolumn{4}{c}{CIFAR10}
		&\multicolumn{4}{c}{CIFAR100}\\ 
		\cmidrule(lr){2-5} 
		\cmidrule(lr){6-9} 
		Method & \multicolumn{2}{c}{Dirichlet-0.1} &  \multicolumn{2}{c}{Dirichlet-0.6} &
		\multicolumn{2}{c}{Dirichlet-0.1} &  \multicolumn{2}{c}{Dirichlet-0.6} \\
		\cmidrule(lr){2-3} \cmidrule(lr){4-5} 
		\cmidrule(lr){6-7} \cmidrule(lr){8-9} 
		& Acc.(\%) & Rounds. & Acc.(\%)& Rounds
		& Acc.(\%) & Rounds & Acc.(\%)& Rounds\\
		& 300R & 30\%  &300R & 30\%
		& 300R & 10\%  &300R & 10\%\\
		\midrule
		DP-FedAvg  \cite{McMahan2018learning}  &38.69 $\pm0.55$ &108 (1$\times$) &53.86 $\pm0.12$ &36 (1$\times$) &14.89 $\pm0.33$ &117 (1$\times$) &20.09 $\pm0.78$ &66 (1$\times$)  \\
		
        Fed-AMP \cite{Rui2022Federated} &39.44 $\pm0.21$ &109 (1$\times$) &54.17 $\pm0.65$ &36 (1$\times$) &15.06 $\pm0.09$ &125 (0.9$\times$) &19.72 $\pm0.44$ &67 (1$\times$) \\
        
        DP-FedAvg-BLUR \cite{cheng2022differentially} &39.85 $\pm0.47$ &90 (1.2$\times$) &54.19 $\pm0.29$ &36 (1$\times$) &14.86 $\pm0.53$ &133 (0.9$\times$) &20.05 $\pm0.17$ &64 (1$\times$) \\
        
        DP-FedAvg-LS \cite{DP-Fed-LS} &39.33 $\pm0.62$ &109 (1$\times$) &54.09 $\pm0.38$ &34 (1.1$\times$) &14.95 $\pm0.71$ &125 (0.9$\times$) &20.04 $\pm0.06$ &68 (1$\times$) \\
        
        DP-FedSAM \cite{DP_FedSAM} &30.07 $\pm0.15$ &223 (0.5$\times$) &44.50 $\pm0.80$ &106 (0.3$\times$) &11.34 $\pm0.24$ &225 (0.5$\times$) &14.67 $\pm0.57$ &124 (0.5$\times$) \\
        
        \rowcolor{LightBlue}
        \textbf{DP-FedPGN (ours)} &53.37 $\pm0.43$ &54 (2$\times$) &69.39 $\pm0.19$ &17 (2.1$\times$) &22.27 $\pm0.66$ &47 (2.5$\times$) &30.46 $\pm0.31$ &26 (2.5$\times$) \\
        \rowcolor{LightRed}
        \textbf{DP-FedPGN-LS (ours)} &\textbf{55.63} $\pm0.07$ & \textbf{34} (3.2$\times$) & \textbf{70.72} $\pm0.50$ &\textbf{12} (3$\times$) &\textbf{23.32} $\pm0.35$ &\textbf{34} (3.4$\times$) &\textbf{30.58} $\pm0.22$ &\textbf{22} (3$\times$) \\
        \midrule[1.5pt]
	\end{tabular}
	\label{table 2}
\end{table*}

\subsection{Convergence Results}
\begin{theorem}[Convergence for non-convex functions]	
	Under Assumption 1, 3, and 4, if we take $g^0=0$,
	$$
	\begin{aligned}
		& \beta=\min \left\{1, \sqrt{\frac{S K L \Delta}{\sigma_l^2 R}}\right\}, \quad \gamma=\min \left\{\frac{1}{24 L}, \frac{\beta}{6 L}\right\}, \\
		& \eta K L \!\lesssim \! \min  \! \!\left\{ \!1, \frac{1}{\beta \gamma L R},\left(\frac{L\Delta}{G_0 \beta^3 R}\right)^{1/2} \! \! \! \!,  \! \frac{1}{(\beta N)^{1/2}}, \frac{1}{\left(\beta^3 N K\right)^{1/4}} \!\right\},
	\end{aligned}
	$$
	then DP-FedPGN converges as
	\begin{equation}
	\frac{1}{R} \sum_{r=0}^{R-1} \mathbb{E}\left[\left\|\nabla f\left(\tilde{x}^{r}\right)\right\|^2\right] \lesssim \mathcal{O}\left(\sqrt{\frac{L \Delta \sigma_l^2}{S K R}}+\frac{L \Delta}{R}+\frac{\sigma^2 C^2}{S^2}\right) .
	\end{equation}
	Here $G_0:=\frac{1}{N} \sum_{i=1}^N\left\|\nabla f_i\left(\tilde{x}^0\right)\right\|^2$.
\end{theorem}

To begin with, we first investigate the convergence performance of DP-FedPGN.  See its proof in Appendix.  In Table \ref{table com}, the convergence speed of DP-FedPGN is faster than that of DP-FedSAM, and we do not need data heterogeneity Assumption (Assumption 2). And in terms of computational complexity, we also require less overhead than DP-FedSAM.
\subsection{Sensitivity Analysis}

To achieve client-level privacy protection, we derive the sensitivity of the aggregation process at first after clipping the local updates.

\begin{lemma}
  The sensitivity of client-level DP in DP-FedPGN can be expressed as $C / S$.
  
Proof. Given two adjacent batches $\mathcal{W}^t$ and $\tilde{\mathcal{W}}^{t}$, that $\tilde{\mathcal{W}}^{t}$ has one more or less client, we have
\begin{equation}
\left\|\frac{1}{S} \sum_{i \in \mathcal{W}^t} \Delta_i^t-\frac{1}{S} \sum_{j \in \tilde{\mathcal{W}}^{t}} \Delta_j^t\right\|_2=\frac{1}{S}\left\|\tilde{\Delta}_{j^{\prime}}^t\right\|_2 \leq \frac{C}{S},
\end{equation}
where $\tilde{\Delta}_{j^{\prime}}^t$ is the local update of one more or less client.  
\end{lemma}

\subsection{Privacy Analysis}
After adding Gaussian noise, we calculate the cumulative privacy budget along with training as follows.
\begin{theorem} After $R$ communication rounds, the accumulative privacy budget is calculated by:
\begin{equation}\label{epsilon}
\epsilon=\bar{\epsilon}+\frac{(\alpha-1) \log \left(1-\frac{1}{\alpha}\right)-\log (\alpha)-\log (\delta)}{\alpha-1},
\end{equation}
where
\begin{equation}
\bar{\epsilon}=\frac{R}{\alpha-1} \ln \mathbb{E}_{z \sim \mu_0(z)}\left[\left(1-q+\frac{q \mu_1(z)}{\mu_0(z)}\right)^\alpha\right],
\end{equation}
and $q$ is sample rate for client selection, $\mu_0(z)$ and $\mu_1(z)$ denote the Gaussian probability density function (PDF) of $\mathcal{N}(0, \sigma)$ and the mixture of two Gaussian distributions $q \mathcal{N}(1, \sigma)+(1-q) \mathcal{N}(0, \sigma)$, respectively, $\sigma$ is the noise STD, $\alpha$ is a selectable variable.
\end{theorem}


\begin{figure*}[tb]
\centering
\begin{minipage}[c]{0.245\textwidth}
\centering
\subcaptionbox{ResNet-10, CIFAR10, $\alpha\!=\!0.1$}{\includegraphics[width=\textwidth]{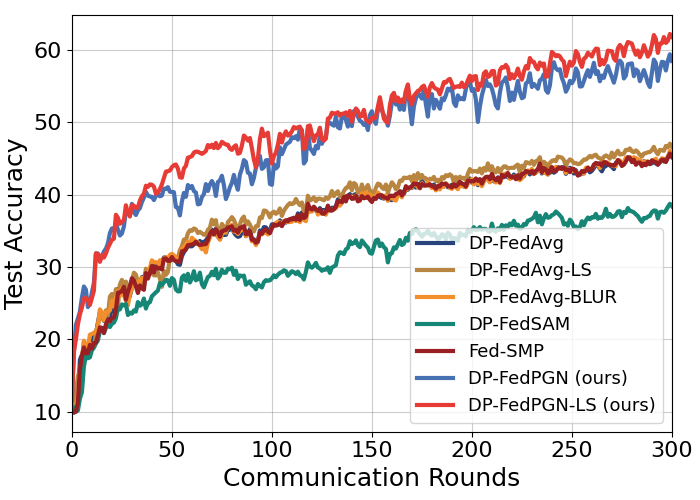}}
\end{minipage}
\begin{minipage}[c]{0.245\textwidth}
\centering
\subcaptionbox{ResNet-10, CIFAR100, $\alpha\!=\!0.1$}{\includegraphics[width=\textwidth]{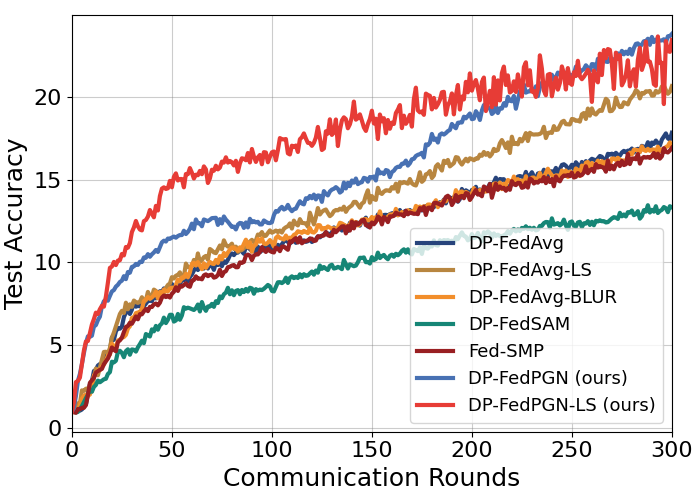}}
\end{minipage}
\begin{minipage}[c]{0.245\textwidth}
\centering
\subcaptionbox{ResNet-10, CIFAR10, $\alpha\!=\!0.6$}{\includegraphics[width=\textwidth]{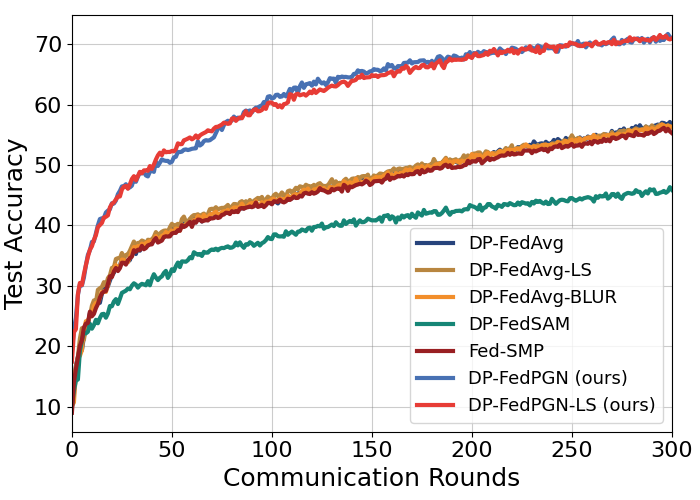}}
\end{minipage}
\begin{minipage}[c]{0.245\textwidth}
\centering
\subcaptionbox{ResNet-10, CIFAR100, $\alpha\!=\!0.6$}{\includegraphics[width=\textwidth]{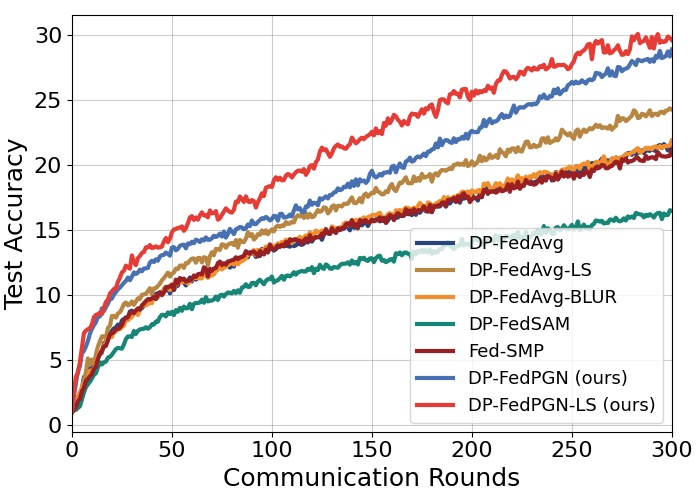}}
\end{minipage}
\centering
\begin{minipage}[c]{0.245\textwidth}
    \centering
    \subcaptionbox{ResNet-18, CIFAR10, $\alpha\!=\!0.1$}{\includegraphics[width=\textwidth]{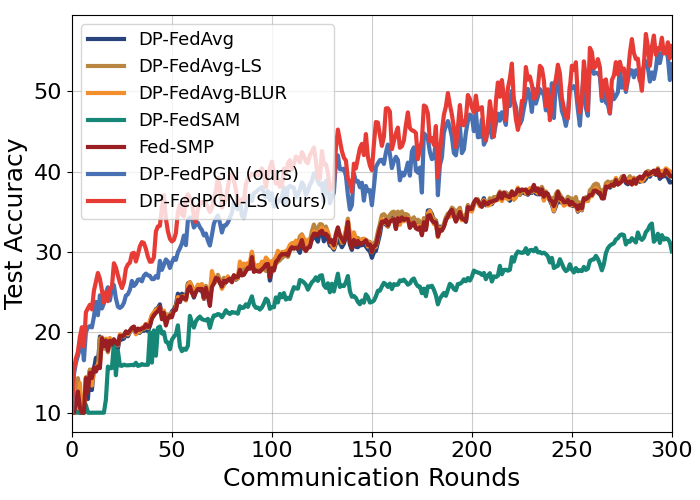}}
\end{minipage}
\begin{minipage}[c]{0.245\textwidth}
    \centering
    \subcaptionbox{ResNet-18, CIFAR100, $\alpha\!=\!0.1$}{\includegraphics[width=\textwidth]{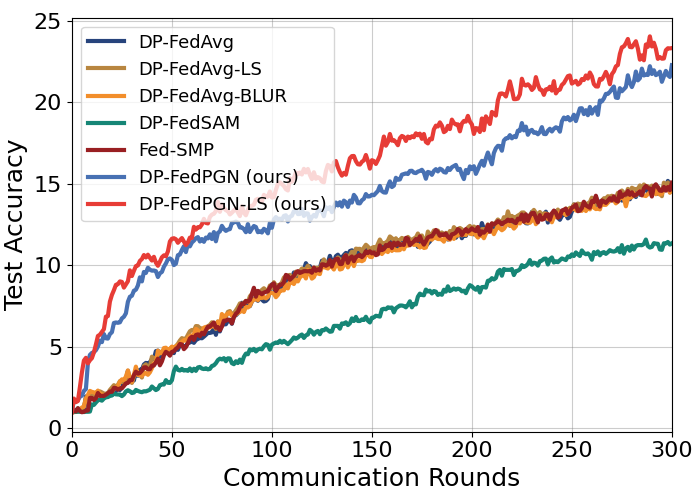}}
\end{minipage}
\begin{minipage}[c]{0.245\textwidth}
    \centering
    \subcaptionbox{ResNet-18, CIFAR10, $\alpha\!=\!0.6$}{\includegraphics[width=\textwidth]{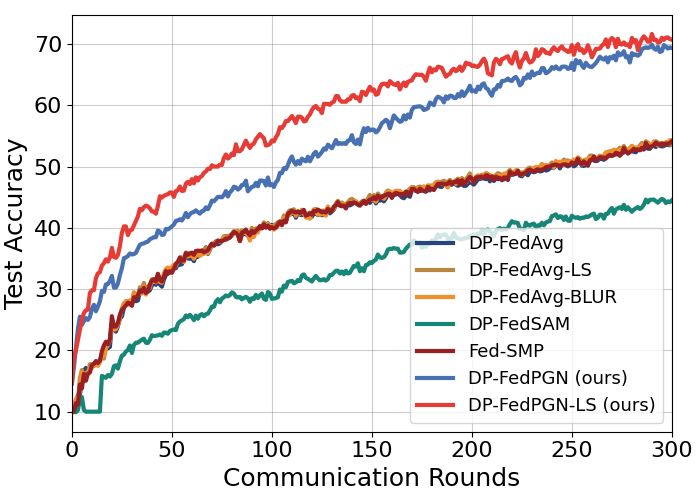}}
\end{minipage}
\begin{minipage}[c]{0.245\textwidth}
    \centering
    \subcaptionbox{ResNet-18, CIFAR100, $\alpha\!=\!0.6$}{\includegraphics[width=\textwidth]{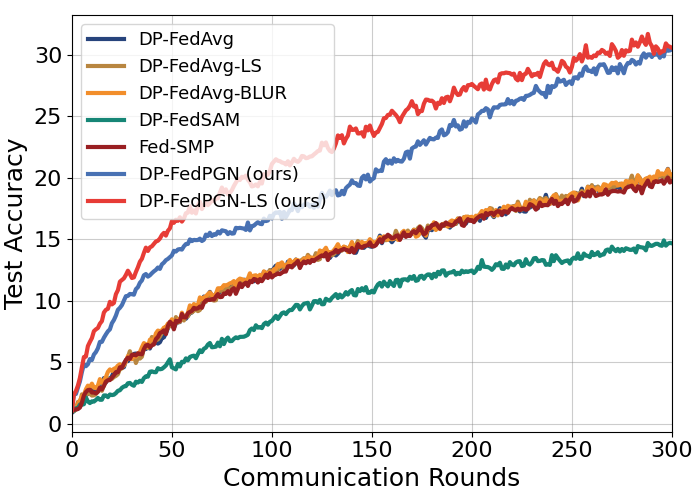}}
\end{minipage}
    \caption{Convergence plots  on CIFAR10 and CIFAR100 (Dirichlet-0.1 and Dirichlet-0.6) with ResNet-18 and ResNet-10.}
    \label{figure 5}
\end{figure*}

 \section{Experiments}

 \begin{table*}[tb]
    \setlength{\tabcolsep}{3pt}
    \centering
    \caption{Performance comparison under different privacy budgets $\epsilon$ on CIFAR10 and CIFAR100 with ResNet-18.}
    \begin{tabular}{m{1.5cm} >{\raggedright}m{3.5cm}  >{\centering\arraybackslash}p{2cm} 
>{\centering\arraybackslash}p{2cm} 
>{\centering\arraybackslash}p{2cm} 
>{\centering\arraybackslash}p{2cm}
>{\centering\arraybackslash}p{2cm}} 
        \toprule
        \multirow{2}*{Task}
        & \multirow{2}*{Method} & \multicolumn{5}{c}{ Averaged test accuracy (\%)} \\
        \cmidrule(lr){3-7}
        & & $\epsilon = 4$ & $\epsilon = 6$ & $\epsilon = 8$ & $\epsilon = 10$& $\epsilon = \infty$ \\
        \midrule
        \multirow{7}{*}{CIFAR10} 
        & DP-FedAvg \cite{McMahan2018learning}  & 29.43 $\pm0.35$ & 38.05 $\pm0.62$ & 42.38 $\pm0.17$ & 46.91 $\pm0.44$ &80.87 $\pm0.73$\\
        & Fed-SMP \cite{Rui2022Federated}  & 29.57 $\pm0.09$ & 38.27 $\pm0.78$ & 42.73 $\pm0.53$ & 47.21 $\pm0.21$ &79.58 $\pm0.15$\\
        & DP-FedAvg-BLUR \cite{cheng2022differentially}  & 29.30 $\pm0.66$ & 38.11 $\pm0.33$ & 42.80 $\pm0.08$ & 46.59 $\pm0.72$ &80.51 $\pm0.68$\\
        & DP-FedAvg-LS  \cite{DP-Fed-LS}   & 29.78 $\pm0.25$ & 38.69 $\pm0.15$ & 42.83 $\pm0.49$ & 47.44 $\pm0.57$ &82.93 $\pm0.24$\\
        & DP-FedSAM  \cite{DP_FedSAM} & 21.77 $\pm0.40$ & 28.99 $\pm0.03$ & 32.32 $\pm0.65$ & 38.24 $\pm0.29$ &77.45 $\pm0.61$\\
        \rowcolor{LightBlue}
        & \textbf{DP-FedPGN (ours)}   & 37.41 $\pm0.12$ & 46.02 $\pm0.80$ & 53.29 $\pm0.36$ & 61.09 $\pm0.07$ &85.10 $\pm0.19$\\
        \rowcolor{LightRed}
        & \textbf{DP-FedPGN-LS (ours)}  & \textbf{42.89} $\pm0.55$ & \textbf{51.99} $\pm0.22$ & \textbf{60.09} $\pm0.41$ & \textbf{65.00} $\pm0.18$ &\textbf{87.53} $\pm0.33$\\
        \midrule
        \multirow{7}{*}{CIFAR100}
        & DP-FedAvg \cite{McMahan2018learning} & 5.60 $\pm0.63$  & 11.00 $\pm0.34$ & 13.48 $\pm0.77$ & 16.09 $\pm0.05$ &45.13 $\pm0.57$\\
        & Fed-SMP \cite{Rui2022Federated} & 5.64 $\pm0.28$  & 10.55 $\pm0.50$ & 13.79 $\pm0.13$ & 16.23 $\pm0.67$ &44.60 $\pm0.42$\\
        & DP-FedAvg-BLUR \cite{cheng2022differentially} & 6.21 $\pm0.45$  & 11.24 $\pm0.09$ & 13.94 $\pm0.71$ & 16.03 $\pm0.39$ &44.70 $\pm0.09$\\
        & DP-FedAvg-LS  \cite{DP-Fed-LS}   & 5.44 $\pm0.16$  & 10.99 $\pm0.58$ & 13.79 $\pm0.26$ & 16.08 $\pm0.73$ &47.73 $\pm0.77$\\
        & DP-FedSAM \cite{DP_FedSAM} & 3.39 $\pm0.80$  & 6.83 $\pm0.19$  & 10.00 $\pm0.47$ & 12.23 $\pm0.32$ &47.07 $\pm0.38$\\
        \rowcolor{LightBlue}
        & \textbf{DP-FedPGN (ours)}  & 11.65 $\pm0.54$ & 15.65 $\pm0.61$ & 18.96 $\pm0.25$ & 23.71 $\pm0.54$ &54.14 $\pm0.22$\\
        \rowcolor{LightRed}
        & \textbf{DP-FedPGN-LS (ours)}   & \textbf{13.41} $\pm0.37$ & \textbf{19.06} $\pm0.10$ & \textbf{22.19} $\pm0.66$ & \textbf{26.60} $\pm0.14$ &\textbf{59.74} $\pm0.65$\\
        \midrule[1.5pt]
    \end{tabular}
    \label{table 3}
\end{table*}

\begin{table*}[tb]
	\centering
	\setlength{\tabcolsep}{12pt}
	\caption{Comparison  accuracy (\%) of each algorithm on ViT-Base and Roberta-Base in 1,00 communication rounds.}
	\begin{tabular}{llcccccccc}
		\midrule[1.5pt]
		&\multicolumn{2}{c}{ViT-Babse}
		&\multicolumn{4}{c}{Roberta-Base}\\ 
		\cmidrule(lr){2-3} 
		\cmidrule(lr){4-7} 
		\cmidrule(lr){2-3} \cmidrule(lr){4-7} 
		& Acc.(\%) & Acc.(\%) & Acc.(\%)& Acc.(\%)
		& Acc.(\%) & Acc.(\%) \\
		& CIFAR10 & CIFAR100 &SST-2 & QQP
		& MRPC &QNLI  \\
		\midrule		DP-FedAvg \cite{McMahan2018learning}  &83.09$\pm0.32$ &59.39$\pm0.24$ &93.44$\pm0.33$ &86.81$\pm0.25$ & 85.82$\pm0.24$ & 83.34$\pm0.18$\\
		
		Fed-SMP \cite{Rui2022Federated}   &83.56$\pm0.45$ &59.27$\pm0.32$ &93.54$\pm0.25$ &86.32$\pm0.13$&86.27$\pm0.11$  & 83.12$\pm0.32$   \\
		
		DP-FedAvg-BLUR \cite{cheng2022differentially} &83.62$\pm0.22$ &59.18 $\pm0.19$&93.66$\pm0.14$ &86.94$\pm0.33$&86.43$\pm0.34$  &83.46$\pm0.18$   \\
		
		DP-FedAvg-LS \cite{DP-Fed-LS}   &85.12$\pm0.52$&62.16$\pm0.17$&94.12$\pm0.17$ &88.31$\pm0.12$&87.22$\pm0.15$  &85.16$\pm0.11$ \\
		
		DP-FedSAM \cite{DP_FedSAM} &81.35$\pm0.24$ &57.22$\pm0.28$ &93.88$\pm0.29$&87.11$\pm0.32$&86.17$\pm0.16$  &83.16$\pm0.21$\\
		
		\rowcolor{LightBlue}
		\textbf{DP-FedPGN (ours)}   &86.33$\pm0.64$&65.23$\pm0.32$ &94.35$\pm0.16$ &89.73$\pm0.55$&89.12$\pm0.15$  &87.22$\pm0.31$ \\
        \rowcolor{LightRed}
		\textbf{DP-FedPGN-LS (ours)}   &\textbf{88.23}$\pm0.35$ & \textbf{67.32}$\pm0.12$   & \textbf{94.66$\pm0.19$} &\textbf{90.65$\pm0.23$} &\textbf{89.22$\pm0.19$ } &\textbf{90.66$\pm0.16$}   \\
		\midrule[1.5pt]
	\end{tabular}
	\label{table 4}
\end{table*}

\subsection{Experimental Settings}

\textbf{{Datasets.}} We evaluate our algorithms on the CIFAR10 \cite{krizhevsky2009learning}, CIFAR100 \cite{krizhevsky2009learning} datasets. We also tested the NLP dataset, such as SST-2 \cite{socher2013recursive}, QQP \cite{socher2013recursive}, MRPC \cite{dolan2005automatically}, and QNLI \cite{rajpurkar2018know}. For a non-IID data setup, we simulate the data heterogeneity by sampling the label ratios from a Dirichlet distribution \cite{hsu2019measuring}. $\operatorname{Dir}(\alpha)$ with parameters $\alpha=0.1$ and $\alpha=0.6$. Dirichlet-0.1 implies very high data heterogeneity and Dirichlet-0.6 implies low data heterogeneity.

\textbf{{Models.}} We use standard models including  GNResNet-10 (Group Norm ResNet-10), GNResNet-18 (Group Norm ResNet-18) \cite{he2016deep}, Vision Transformer (ViT-Base) \cite{dosovitskiy2020image}, and Roberta-Base.  

\textbf{{Baselines.}} We compare DP-FedPGN with many 
SOTA DPFL baselines, including DP-FedAvg \cite{McMahan2018learning}, DP-FedAvg-BLUR \cite{cheng2022differentially}, Fed-SMP \cite{Rui2022Federated}, DP-FedAvg-LS \cite{DP-Fed-LS}, DP-FedSAM \cite{DP_FedSAM}.

\textbf{Hyper-parameter Settings.}
The number of clients is 500. batch size $B\!=\!50$, local epoch $E\!=\!5$, $K \!=\! 50$  and client participating $ 10\%$. We set the search range on the grid of the client learning rate by $\eta \!\in\!\{10^{-3}, 3\times 10^{-3},...,10^{-1}, 3\times 10^{-1}\}$. The decay of the learning rate per round is $0.998$, total $R=300$. 
Specifically, we set $\rho\!=\!0.2$, $\beta=0.3$, $\gamma=\eta K$ for DP-FedPGN and DP-FedPGN-LS. We set the Laplacian smoothing coefficient $\sigma_{ls}$ of DP-FedPGN-LS to 0.01 and we set the Laplacian smoothing coefficient $\sigma_{ls}$ of DP-FedPGN-LS in $\{10^{-2}, 3\times 10^{-2},...,10^{-1}, 3\times 10^{-1}, 5\times 10^{-1}\}$. We set $\rho\!\in\!\{10^{-2}, 3\times 10^{-2},...,10^{-1}, 3\times 10^{-1}, 5\times 10^{-1}\}$ for DP-FedSAM. For privacy parameters, noise multiplier $\sigma$ is set to 0.8 and the privacy failure probability $\delta=\frac{1}{N}$. The clipping threshold $C$ adopts median clipping. We run each experiment 5 trials and report the average testing accuracy in each experiment.

\subsection{Results on Convolutional Neural Network}
   
\textbf{Performance with compared baselines.} From Table \ref{table 1}, Table \ref{table 2} (Figure \ref{figure 6}), we have the following observations:
 It is seen that our proposed algorithms outperform other baselines under symmetric noise, both on accuracy and generalization perspectives. This means that we significantly improve the performance and generate a better trade-off between performance and privacy in DPFL. For instance, the averaged testing accuracy is $58.24\%$ in DP-FedPGN and $61.80\%$ in DP-FedPGN-LS on CIFAR10 with GNResNet-10, $\alpha=0.1$, which are better than other baselines. Meanwhile, the test accuracy is $ 23.78\%$ in DP-FedPGN and $23.42\%$ in DP-FedPGN-LS on CIFAR100 with GNResNet-10, $\alpha=0.1$.  DP-FedPGN-LS only needs $21$  communication rounds to $10\%$ accuracy (3.4$\times$ acceleration), which is better than DP-FedPGN ($34$ communication rounds) and other baselines. That means our algorithms significantly mitigate the performance degradation issue caused by DP. 

\textbf{Performance under different privacy budgets $\epsilon$.} Table \ref{table 3}
shows the test accuracies for the different levels of privacy guarantees. Note that $\epsilon$ is not a hyperparameter but can be obtained by the privacy design in each round, such as Eq. (\ref{epsilon}). Our methods consistently outperform the previous SOTA methods under various privacy budgets $\epsilon$. Specifically, DP-FedPGN and DP-FedPGN-LS significantly improve the accuracy of DP-FedAvg and DP-FedAvg-LS under the same $\epsilon$, respectively. Furthermore, the test accuracy improves as the privacy budget $\epsilon$ increases, suggesting that a better balance is necessary between training performance and privacy. There is a better trade-off when achieving better accuracy under the same $\epsilon$ or a smaller $\epsilon$ under approximately the same accuracy. 

 \begin{figure}[tb]
    \centering
    \includegraphics[scale=0.3]{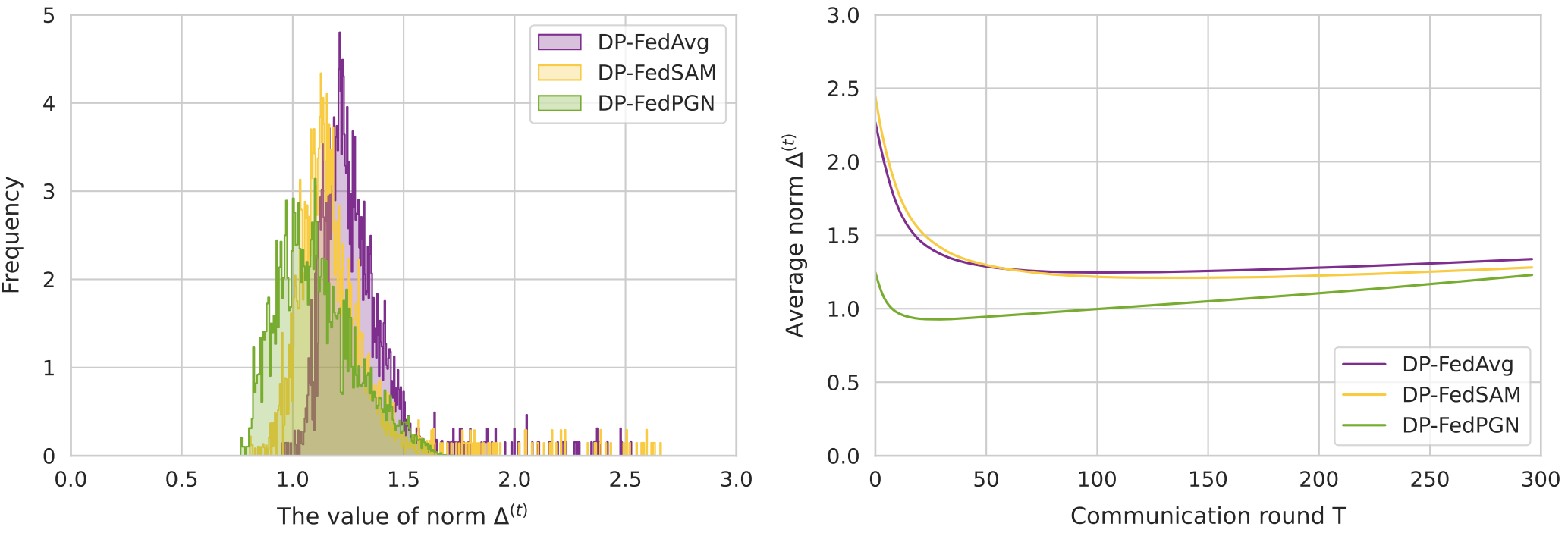}
    \caption{Comparison of norm distribution and average norm of local updates for DP-FedAvg, DP-FedSAM, DP-FedPGN.}
    \label{figure 6}
\end{figure}
        \begin{figure*}[tb]
        \centering
        \begin{minipage}[b]{0.245\textwidth}
            \centering
            \subcaptionbox{Impact of $E$}{\includegraphics[width=\textwidth]{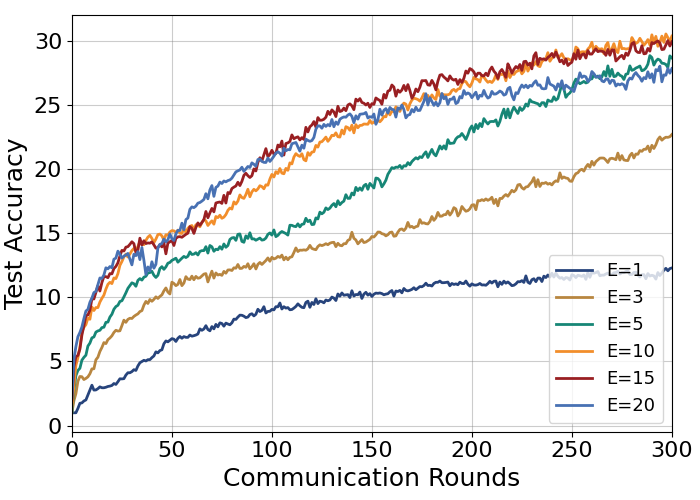}}
        \end{minipage}
        \begin{minipage}[b]{0.245\textwidth}
            \centering
            \subcaptionbox{Impact of $\rho$}{\includegraphics[width=\textwidth]{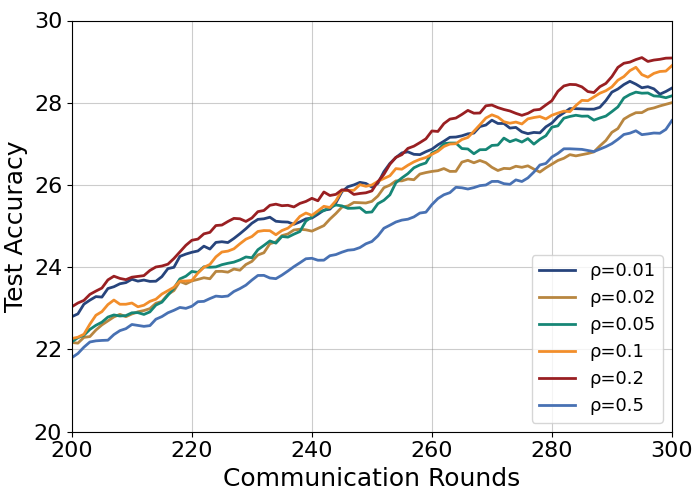}}
        \end{minipage}
         \begin{minipage}[b]{0.245\textwidth}
            \centering
            \subcaptionbox{Impact of $\sigma_{ls}$}
        {\includegraphics[width=\textwidth]{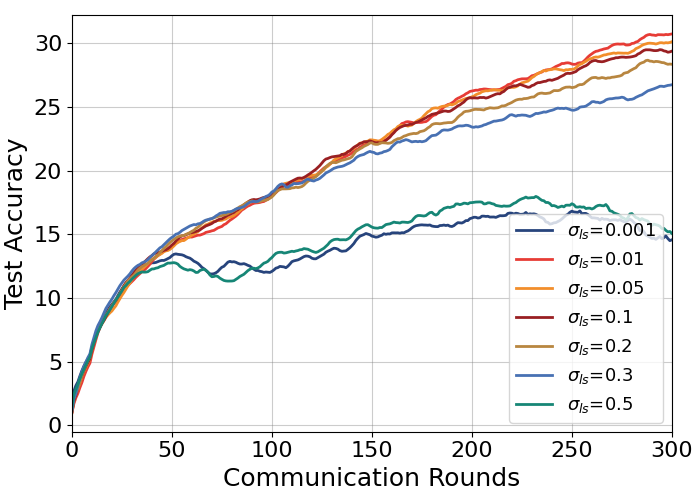}}
        \end{minipage}
        \begin{minipage}[b]{0.245\textwidth}
            \centering
            \subcaptionbox{Impact of $\beta$}{\includegraphics[width=\textwidth]{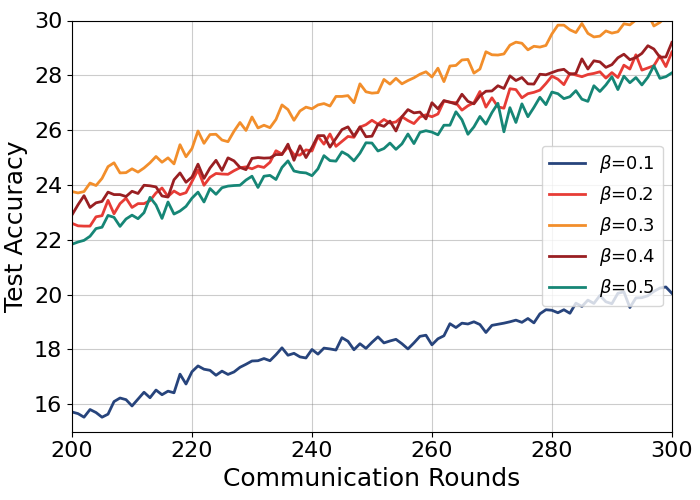}}
        \end{minipage}
            \caption{Convergence plots for DP-FedPGN with different $E$, $\rho$, $\sigma_{ls}$ and $\gamma$ on CIFAR100 with ResNet-10 and Dirichlet-0.6.}
            \label{figure 7}
        \end{figure*}
\subsection{Results on ViT-Base and Roberta-Base}
To verify the superiority of our algorithm in training large models, we used ViT-Base training on CIFAR10 and CIFAR100. We use the pre-trained model with officially provided weights for pre-training on ImageNet-22k. From the experimental results in Table  \ref{table 4}, we can observe that the DP-FedPGN algorithm can achieve the best results We trained using a pre-trained large language model, Roberta-Base to fine-tune the datasets, such as SST-2, QQP, MRPC, and QNLI, and our method achieved the best results on each dataset in Table  \ref{table 4}.
For the training of these models, we used a learning rate of 0.1, a learning rate decay of $0.98$ per round, a batch size of 16, and 100 communication rounds of training in Dirichlet-0.1. This can verify that the DP-FedPGN model can achieve excellent acceleration effects on both the Vision Transformer model and big datasets.

\begin{figure}[tb]
	\centering
	\begin{minipage}[t]{0.23\textwidth}
		\centering
		\subcaptionbox{DP-FedPGN}{\includegraphics[width=\textwidth]{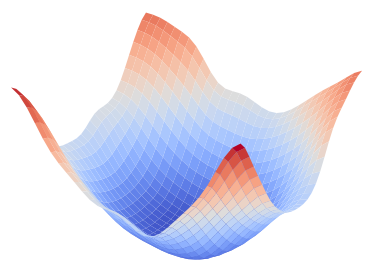}}
	\end{minipage}
	\begin{minipage}[t]{0.23\textwidth}
		\centering
		\subcaptionbox{DP-FedPGN-LS}{\includegraphics[width=\textwidth]{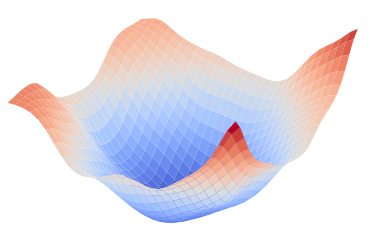}}
	\end{minipage}
		\caption{Loss landscapes of DP-FedPGN and DP-FedPGN-LS on CIFAR100 with ResNet-10 and Dirichlet-0.1.}
		\label{figure 8}
\end{figure}       
	
\subsection{Discussion for PGN and LS in DPFL}
In this subsection, we empirically discuss how PGN mitigates the negative impacts of DP from the aspects of the
local update norm, the visualization of the landscape and contour of the loss.

\textbf{The norm of local update.} 
To validate the theoretical results for mitigating the adverse impacts of the norm of local updates, we conduct experiments on DP-FedPGN, DP-FedSAM and DP-FedAvg with clipping as shown in Figure \ref{figure 6}. We show the norm $\Delta_i^t$ distribution and average norm $\bar{\Delta}^t$ of local updates before clipping during the communication rounds. In contrast to DP-FedAvg and DP-FedSAM, most of the norm of DP-FedPGN is distributed at the smaller value in our scheme in Figure \ref{figure 6} (a), which means that the clipping operation drops less information. Meanwhile, the on-average norm $\bar{\Delta}^t$ is smaller than others as shown in Figure \ref{figure 6} (b).

\textbf{Laplacian smoothing.}
We can observe from Figure \ref{figure 8} that Laplacian Smoothing can make the aggregated global model flatter.

\subsection{Ablation Study}

\textbf{Perturbation weight $\rho$.}  Perturbation weight $\rho$ has an impact on performance as the added perturbation is accumulated when the communication round $R$ increases. To select a proper value for our algorithms, we conduct some experiments on various perturbation radius from the set $\{0.01,0.02,0.05,0.1,0.2,0.5\}$ in Figure \ref{figure 7} (b). As $\rho=0.2$, we achieve better convergence and performance.

\textbf{Impact of $\beta$.} 
We test DP-FedPGN with $\beta$ taken values in the set $\{0.1,0.2,0.3,0.4,0.5\}$, training ResNet-10 on CIFAR100 datasets with Dirichlet-0.6 split on 500 clients, $ 10\% $ participating setting. The test accuracies and the convergence plots are provided in Figure \ref{figure 7} (d). As shown in  Figure \ref{figure 7} (d),  we find that performance is best when setting $\beta$ to about 0.3.

\textbf{Impact of Local iteration steps $K$: }  Large local iteration steps $K=E*10$ can help the convergence in DPFL. 
In Figure \ref{figure 7} (b), our algorithm can accelerate the convergence in Theorem 1 as a larger $K$ is adopted, that is, use a larger epoch value. However, the adverse impact of clipping on training increases as $K$ is too large, e.g., $E=20$. Thus, we choose $E=5$ to balance it.

\textbf{Laplacian smoothing coefficient $\sigma_{ls}$: }We observed through experiments that too small $\sigma_{ls}$ cannot achieve the smoothing effect, while too large $\sigma_{ls}$ causes significant loss of details in  model. Therefore, we found that choosing $\sigma_{ls}=0.01$ has the best effect.


\section{Conclusion}

In this paper, we focus on the issue of sharp landscape loss caused by gradient cropping and noise in DPFL. We first conducted extensive experiments and found that the loss landscape of DPFL is very sharp, and rethought the limitations of DP-FedSAM, whose global model may not necessarily be flat. Therefore, we proposed a novel DP-FedPGN algorithm from the perspective of global gradient norm penalty. We analyzed in detail how DP-FedPGN achieves global flatness and mitigates the adverse effects of DP, and demonstrated its stricter convergence speed. Moreover, we also found that Laplacian smoothing can make the global model flatter and have better generalization performance. Finally, the experimental results of visual and natural language processing tasks validated the superiority of our method.       
		

\bibliographystyle{IEEEtran}
\bibliography{sample-base}

\begin{thebibliography}{10}
\providecommand{\url}[1]{#1}
\csname url@samestyle\endcsname
\providecommand{\newblock}{\relax}
\providecommand{\bibinfo}[2]{#2}
\providecommand{\BIBentrySTDinterwordspacing}{\spaceskip=0pt\relax}
\providecommand{\BIBentryALTinterwordstretchfactor}{4}
\providecommand{\BIBentryALTinterwordspacing}{\spaceskip=\fontdimen2\font plus
\BIBentryALTinterwordstretchfactor\fontdimen3\font minus
  \fontdimen4\font\relax}
\providecommand{\BIBforeignlanguage}[2]{{%
\expandafter\ifx\csname l@#1\endcsname\relax
\typeout{** WARNING: IEEEtran.bst: No hyphenation pattern has been}%
\typeout{** loaded for the language `#1'. Using the pattern for}%
\typeout{** the default language instead.}%
\else
\language=\csname l@#1\endcsname
\fi
#2}}
\providecommand{\BIBdecl}{\relax}
\BIBdecl

\bibitem{McMahan2018learning}
H.~B. McMahan, D.~Ramage, K.~Talwar, and L.~Zhang, ``Learning differentially
  private recurrent language models,'' in \emph{Proc. International Conference
  on Learning Representations (ICLR)}, Apr. 2018.

\bibitem{DP_FedSAM}
Y.~Shi, Y.~Liu, K.~Wei, L.~Shen, X.~Wang, and D.~Tao, ``Make landscape flatter
  in differentially private federated learning,'' in \emph{2023 IEEE/CVF
  Conference on Computer Vision and Pattern Recognition (CVPR)}, 2023, pp.
  24\,552--24\,562.

\bibitem{mcmahan2017communication}
B.~McMahan, E.~Moore, D.~Ramage, S.~Hampson, and B.~A. y~Arcas,
  ``Communication-efficient learning of deep networks from decentralized
  data,'' in \emph{Artificial intelligence and statistics}.\hskip 1em plus
  0.5em minus 0.4em\relax PMLR, 2017, pp. 1273--1282.

\bibitem{kairouz2021advances}
P.~Kairouz, H.~B. McMahan, B.~Avent, A.~Bellet, M.~Bennis, A.~N. Bhagoji,
  K.~Bonawitz, Z.~Charles, G.~Cormode, R.~Cummings \emph{et~al.}, ``Advances
  and open problems in federated learning,'' \emph{Foundations and
  trends{\textregistered} in machine learning}, vol.~14, no. 1--2, pp. 1--210,
  2021.

\bibitem{10645291}
Z.~Li, Z.~Chen, X.~Wei, S.~Gao, H.~Yue, Z.~Xu, and T.~Q. Quek, ``Exploiting
  complex network-based clustering for personalization-enhanced hierarchical
  federated edge learning,'' \emph{IEEE Transactions on Mobile Computing},
  vol.~23, no.~12, pp. 14\,852--14\,870, 2024.

\bibitem{11219205}
H.~Wang, B.~Li, P.~Chen, L.~Wu, Z.~Li, and T.~Q. Quek, ``Lbkd: Rethinking
  federated backdoors for low-altitude economy via llms and bidirectional
  knowledge distillation,'' \emph{IEEE Transactions on Network Science and
  Engineering}, pp. 1--18, 2025.

\bibitem{LI2026108038}
Z.~Li, B.~Li, K.~Zhang, B.~Wei, H.~Liu, Z.~Chen, X.~Xie, and T.~Q. Quek,
  ``Heterogeneity-aware high-efficiency federated learning with hybrid
  synchronous-asynchronous splitting strategy,'' \emph{Neural Networks}, vol.
  193, p. 108038, 2026.

\bibitem{10566979}
X.~Wei, D.~Jiao, S.~Tong, Z.~Li, C.~Ren, and H.~Yue, ``Cfpa: Cognitive
  federated partial adaptation for effective personalization,'' in \emph{2023
  19th International Conference on Mobility, Sensing and Networking (MSN)},
  2023, pp. 553--561.

\bibitem{cheng2022differentially}
A.~Cheng, P.~Wang, X.~S. Zhang, and J.~Cheng, ``Differentially private
  federated learning with local regularization and sparsification,'' in
  \emph{Proceedings of the IEEE/CVF Conference on Computer Vision and Pattern
  Recognition}, 2022.

\bibitem{geyer2017differentially}
R.~C. Geyer, T.~Klein, and M.~Nabi, ``Differentially private federated
  learning: A client level perspective,'' \emph{arXiv preprint
  arXiv:1712.07557}, 2017.

\bibitem{hu2022federated}
R.~Hu, Y.~Gong, and Y.~Guo, ``Federated learning with sparsified model
  perturbation: Improving accuracy under client-level differential privacy,''
  \emph{arXiv preprint arXiv:2202.07178}, 2022.

\bibitem{KairouzL2021the}
P.~Kairouz, Z.~Liu, and T.~Steinke, ``The distributed discrete gaussian
  mechanism for federated learning with secure aggregation,'' in \emph{Proc.
  International Conference on Machine Learning (ICML)}, Jul. 2021, pp.
  5201--5212.

\bibitem{liu2024fedbcgd}
J.~Liu, F.~Shang, Y.~Liu, H.~Liu, Y.~Li, and Y.~Gong, ``Fedbcgd:
  Communication-efficient accelerated block coordinate gradient descent for
  federated learning,'' in \emph{Proceedings of the 32nd ACM International
  Conference on Multimedia}, 2024, pp. 2955--2963.

\bibitem{liuimproving}
J.~Liu, Y.~Liu, F.~Shang, H.~Liu, J.~Liu, and W.~Feng, ``Improving
  generalization in federated learning with highly heterogeneous data via
  momentum-based stochastic controlled weight averaging,'' in
  \emph{Forty-second International Conference on Machine Learning}.

\bibitem{wu2025llm}
W.~Wu, Z.~Chen, X.~Qiu, S.~Song, X.~Huang, F.~Ma, and J.~Xiao, ``Llm-enhanced
  multimodal fusion for cross-domain sequential recommendation,'' \emph{arXiv
  preprint arXiv:2506.17966}, 2025.

\bibitem{wu2025prompt}
W.~Wu, X.~Qiu, S.~Song, Z.~Chen, X.~Huang, F.~Ma, and J.~Xiao, ``Prompt
  categories cluster for weakly supervised semantic segmentation,'' in
  \emph{Proceedings of the Computer Vision and Pattern Recognition Conference},
  2025, pp. 3198--3207.

\bibitem{10.1145/3746027.3755226}
\BIBentryALTinterwordspacing
J.~Liu, F.~Shang, Y.~Tian, H.~Liu, and Y.~Liu, ``Consistency of local and
  global flatness for federated learning,'' in \emph{Proceedings of the 33rd
  ACM International Conference on Multimedia}, ser. MM '25.\hskip 1em plus
  0.5em minus 0.4em\relax New York, NY, USA: Association for Computing
  Machinery, 2025, p. 3875–3883. [Online]. Available:
  \url{https://doi.org/10.1145/3746027.3755226}
\BIBentrySTDinterwordspacing

\bibitem{yang2024learning}
S.~Yang, J.~Yang, M.~Zhou, Z.~Huang, W.-S. Zheng, X.~Yang, and J.~Ren,
  ``Learning from human educational wisdom: A student-centered knowledge
  distillation method,'' \emph{IEEE Transactions on Pattern Analysis and
  Machine Intelligence}, vol.~46, no.~6, pp. 4188--4205, 2024.

\bibitem{lin2025semantically}
G.~Lin, S.~Yang, W.-S. Zheng, Z.~Li, and Z.~Huang, ``A semantically guided and
  focused network for occluded person re-identification,'' \emph{IEEE
  Transactions on Information Forensics and Security}, vol.~20, pp. 9716--9731,
  2025.

\bibitem{huang2022feature}
Z.~Huang, S.~Yang, M.~Zhou, Z.~Li, Z.~Gong, and Y.~Chen, ``Feature map
  distillation of thin nets for low-resolution object recognition,'' \emph{IEEE
  Transactions on Image Processing}, vol.~31, pp. 1364--1379, 2022.

\bibitem{zhang2022understanding}
X.~Zhang, X.~Chen, M.~Hong, Z.~S. Wu, and J.~Yi, ``Understanding clipping for
  federated learning: Convergence and client-level differential privacy,'' in
  \emph{International Conference on Machine Learning, ICML 2022}, 2022.

\bibitem{Rui2022Federated}
R.~Hu, Y.~Gong, and Y.~Guo, ``Federated learning with sparsified model
  perturbation: Improving accuracy under client-level differential privacy,''
  \emph{CoRR}, 2022.

\bibitem{foret2020sharpness}
P.~Foret, A.~Kleiner, H.~Mobahi, and B.~Neyshabur, ``Sharpness-aware
  minimization for efficiently improving generalization,'' \emph{arXiv preprint
  arXiv:2010.01412}, 2020.

\bibitem{izmailov2018averaging}
P.~Izmailov, D.~Podoprikhin, T.~Garipov, D.~Vetrov, and A.~G. Wilson,
  ``Averaging weights leads to wider optima and better generalization,''
  \emph{arXiv preprint arXiv:1803.05407}, 2018.

\bibitem{kwon2021asam}
J.~Kwon, J.~Kim, H.~Park, and I.~K. Choi, ``Asam: Adaptive sharpness-aware
  minimization for scale-invariant learning of deep neural networks,'' in
  \emph{International Conference on Machine Learning}.\hskip 1em plus 0.5em
  minus 0.4em\relax PMLR, 2021, pp. 5905--5914.

\bibitem{zhao2022penalizing}
Y.~Zhao, H.~Zhang, and X.~Hu, ``Penalizing gradient norm for efficiently
  improving generalization in deep learning,'' in \emph{International
  conference on machine learning}.\hskip 1em plus 0.5em minus 0.4em\relax PMLR,
  2022, pp. 26\,982--26\,992.

\bibitem{osher2022laplacian}
S.~Osher, B.~Wang, P.~Yin, X.~Luo, F.~Barekat, M.~Pham, and A.~Lin, ``Laplacian
  smoothing gradient descent,'' \emph{Research in the Mathematical Sciences},
  vol.~9, no.~3, p.~55, 2022.

\bibitem{wang2020dp}
B.~Wang, Q.~Gu, M.~Boedihardjo, L.~Wang, F.~Barekat, and S.~J. Osher,
  ``Dp-lssgd: A stochastic optimization method to lift the utility in
  privacy-preserving erm,'' in \emph{Mathematical and Scientific Machine
  Learning}.\hskip 1em plus 0.5em minus 0.4em\relax PMLR, 2020, pp. 328--351.

\bibitem{DP-Fed-LS}
Z.~Liang, B.~Wang, Q.~Gu, S.~Osher, and Y.~Yao, ``Differentially private
  federated learning with laplacian smoothing,'' \emph{Applied and
  Computational Harmonic Analysis}, p. 101660, 2024.

\bibitem{hochreiter1997flat}
S.~Hochreiter and J.~Schmidhuber, ``Flat minima,'' \emph{Neural computation},
  vol.~9, no.~1, pp. 1--42, 1997.

\bibitem{keskar2017on}
N.~S. Keskar, D.~Mudigere, J.~Nocedal, M.~Smelyanskiy, and P.~T.~P. Tang, ``On
  large-batch training for deep learning: Generalization gap and sharp
  minima,'' in \emph{International Conference on Learning Representations},
  2017.

\bibitem{neyshabur2017exploring}
B.~Neyshabur, S.~Bhojanapalli, D.~McAllester, and N.~Srebro, ``Exploring
  generalization in deep learning,'' \emph{Advances in neural information
  processing systems}, vol.~30, 2017.

\bibitem{dinh2017sharp}
L.~Dinh, R.~Pascanu, S.~Bengio, and Y.~Bengio, ``Sharp minima can generalize
  for deep nets,'' in \emph{International Conference on Machine
  Learning}.\hskip 1em plus 0.5em minus 0.4em\relax PMLR, 2017, pp. 1019--1028.

\bibitem{virmaux2018lipschitz}
A.~Virmaux and K.~Scaman, ``Lipschitz regularity of deep neural networks:
  analysis and efficient estimation,'' \emph{Advances in Neural Information
  Processing Systems}, vol.~31, 2018.

\bibitem{xie2021a}
Z.~Xie, I.~Sato, and M.~Sugiyama, ``A diffusion theory for deep learning
  dynamics: Stochastic gradient descent exponentially favors flat minima,'' in
  \emph{International Conference on Learning Representations}, 2021.

\bibitem{dwork2014algorithmic}
C.~Dwork, A.~Roth \emph{et~al.}, ``The algorithmic foundations of differential
  privacy,'' \emph{Foundations and Trends{\textregistered} in Theoretical
  Computer Science}, vol.~9, no. 3--4, pp. 211--407, 2014.

\bibitem{mironov2017renyi}
I.~Mironov, ``R{\'e}nyi differential privacy,'' in \emph{2017 IEEE 30th
  computer security foundations symposium (CSF)}.\hskip 1em plus 0.5em minus
  0.4em\relax IEEE, 2017, pp. 263--275.

\bibitem{he2021tighter}
F.~He, B.~Wang, and D.~Tao, ``Tighter generalization bounds for iterative
  differentially private learning algorithms,'' in \emph{Uncertainty in
  Artificial Intelligence}.\hskip 1em plus 0.5em minus 0.4em\relax PMLR, 2021,
  pp. 802--812.

\bibitem{krizhevsky2009learning}
A.~Krizhevsky, G.~Hinton \emph{et~al.}, ``Learning multiple layers of features
  from tiny images,'' 2009.

\bibitem{socher2013recursive}
R.~Socher, A.~Perelygin, J.~Wu, J.~Chuang, C.~D. Manning, A.~Y. Ng, and
  C.~Potts, ``Recursive deep models for semantic compositionality over a
  sentiment treebank,'' in \emph{Proceedings of the 2013 conference on
  empirical methods in natural language processing}, 2013, pp. 1631--1642.

\bibitem{dolan2005automatically}
B.~Dolan and C.~Brockett, ``Automatically constructing a corpus of sentential
  paraphrases,'' in \emph{Third international workshop on paraphrasing
  (IWP2005)}, 2005.

\bibitem{rajpurkar2018know}
P.~Rajpurkar, R.~Jia, and P.~Liang, ``Know what you don't know: Unanswerable
  questions for squad,'' \emph{arXiv preprint arXiv:1806.03822}, 2018.

\bibitem{hsu2019measuring}
T.-M.~H. Hsu, H.~Qi, and M.~Brown, ``Measuring the effects of non-identical
  data distribution for federated visual classification,'' \emph{arXiv preprint
  arXiv:1909.06335}, 2019.

\bibitem{he2016deep}
K.~He, X.~Zhang, S.~Ren, and J.~Sun, ``Deep residual learning for image
  recognition,'' in \emph{Proceedings of the IEEE conference on computer vision
  and pattern recognition}, 2016, pp. 770--778.

\bibitem{dosovitskiy2020image}
A.~Dosovitskiy, L.~Beyer, A.~Kolesnikov, D.~Weissenborn, X.~Zhai,
  T.~Unterthiner, M.~Dehghani, M.~Minderer, G.~Heigold, S.~Gelly \emph{et~al.},
  ``An image is worth 16x16 words: Transformers for image recognition at
  scale,'' \emph{arXiv preprint arXiv:2010.11929}, 2020.

\end{thebibliography}

\newpage
\onecolumn
\appendix
\section{Assumption}

We analyze generalization based on  following assumptions: 
\begin{Assumption}
\label{asp:smooth} \textit{(Smoothness).  $F_i$ is $L$-smooth for all $i \in$ $[N]$, 
	\begin{equation}
	\left\|\nabla F_i(\boldsymbol{x}_{1})-\nabla F_i(\boldsymbol{x}_{2})\right\| \leq L\|\boldsymbol{x}_{1}-\boldsymbol{x}_{2}\|
	\end{equation}
	for all $\boldsymbol{x}_{1}, \boldsymbol{x}_{2}$ in its domain and $i \in[N]$.}
\end{Assumption}
\begin{Assumption}\label{asp:data_var}  \textit{(Bounded variance of data heterogeneity). The global variability of the local gradient of the loss function is bounded by $\sigma_g^2$ for all $i \in[N]$, 
	\begin{equation}
	\left\|\nabla F_i\left(\boldsymbol{x}\right)-\nabla F\left(\boldsymbol{x}\right)\right\|^2 \leq \sigma_g^2
	\end{equation}
	}
\end{Assumption}
\begin{Assumption}
(Bounded variance of stochastic gradient).\label{asp:sgd_var}  The stochastic gradient $\nabla F_i\left(\boldsymbol{x}, \xi_i\right)$, computed by the $i$-th client of model parameter $\boldsymbol{x}$ using mini-batch $\xi_i$, is an unbiased estimator of $\nabla F_i(\boldsymbol{x})$ with variance bounded by $\sigma_l^2$, i.e.,
\begin{equation}
	\mathbb{E}_{\xi_i}\left\|\nabla F_i\left(\boldsymbol{x}, \xi_i\right) - \nabla F_i(\boldsymbol{x})\right\|^2 \leq \sigma_l^2
\end{equation}
for all $i \in [N]$, where the expectation is over all local datasets.
\end{Assumption}
\begin{Assumption}
\label{asp:clip}  There exists a clipping constant $C$ independent of $i, r$ such that $\left\|\bar{\Delta}_i^r\right\| \leq C$.
\end{Assumption}

\section{ Main Lemmas} 
\begin{lemma} \label{lem:bias-var} Suppose $\left\{X_1, \cdots, X_\tau\right\} \subset \mathbb{R}^d$ be random variables that are potentially dependent. If their marginal means and variances satisfy $\mathbb{E}\left[X_i\right]=\mu_i$ and $\mathbb{E}\left[\| X_i-\right.$ $\left.\mu_i \|^2\right] \leq \sigma^2$, then it holds that
	$$
	\mathbb{E}\left[\left\|\sum_{i=1}^\tau X_i\right\|^2\right] \leq\left\|\sum_{i=1}^\tau \mu_i\right\|^2+\tau^2 \sigma^2
	$$
	If they are correlated in the Markov way such that $\mathbb{E}\left[X_i \mid X_{i-1}, \cdots X_1\right]=\mu_i$ and $\mathbb{E}\left[\left\|X_i-\mu_i\right\|^2 \mid\right.$ $\left.\mu_i\right] \leq \sigma^2$, i.e., the variables $\left\{X_i-\mu_i\right\}$ form a martingale. Then the following tighter bound holds:
	$$
	\mathbb{E}\left[\left\|\sum_{i=1}^\tau X_i\right\|^2\right] \leq 2 \mathbb{E}\left[\left\|\sum_{i=1}^\tau \mu_i\right\|^2\right]+2 \tau \sigma^2
	$$
\end{lemma}
\begin{lemma} \label{lem:par_sample} Given vectors $v_1, \cdots, v_N \in \mathbb{R}^d$ and $\bar{v}=\frac{1}{N} \sum_{i=1}^N v_i$, if we sample $\mathcal{S} \subset\{1, \cdots, N\}$ uniformly randomly such that $|\mathcal{S}|=S$, then it holds that
	
	$$
	\mathbb{E}\left[\left\|\frac{1}{S} \sum_{i \in \mathcal{S}} v_i\right\|^2\right]=\|\bar{v}\|^2+\frac{N-S}{S(N-1)} \frac{1}{N} \sum_{i=1}^N\left\|v_i-\bar{v}\right\|^2 .
	$$
\end{lemma}
\begin{proof}
	Letting $\mathbb{I}\{i \in \mathcal{S}\}$ be the indicator for the event $i \in \mathcal{S}_r$, we prove this lemma by direct calculation as follows:
	$$
	\begin{aligned}
		\mathbb{E}\left[\left\|\frac{1}{S} \sum_{i \in \mathcal{S}} v_i\right\|^2\right] & =\mathbb{E}\left[\left\|\frac{1}{S} \sum_{i=1}^N v_i \mathbb{I}\{i \in \mathcal{S}\}\right\|^2\right] \\
		& =\frac{1}{S^2} \mathbb{E}\left[\left(\sum_i\left\|v_i\right\|^2 \mathbb{I}\{i \in \mathcal{S}\}+2 \sum_{i<j} v_i^{\top} v_j \mathbb{I}\{i, j \in \mathcal{S}\}\right)\right] \\
		& =\frac{1}{S N} \sum_{i=1}^N\left\|v_i\right\|^2+\frac{1}{S^2} \frac{S(S-1)}{N(N-1)} 2 \sum_{i<j} v_i^{\top} v_j \\
		& =\frac{1}{S N} \sum_{i=1}^N\left\|v_i\right\|^2+\frac{1}{S^2} \frac{S(S-1)}{N(N-1)}\left(\left\|\sum_{i=1}^N v_i\right\|^2-\sum_{i=1}^N\left\|v_i\right\|^2\right) \\
		& =\frac{N-S}{S(N-1)} \frac{1}{N} \sum_{i=1}^N\left\|v_i\right\|^2+\frac{N(S-1)}{S(N-1)}\|\bar{v}\|^2 \\
		& =\frac{N-S}{S(N-1)} \frac{1}{N} \sum_{i=1}^N\left\|v_i-\bar{v}\right\|^2+\|\bar{v}\|^2 .
	\end{aligned}
	$$
\end{proof}

\begin{lemma}\label{lem:dp-noise} (Regularity of DP-noised stochastic gradients). Under Assumption \ref{asp:clip}, for any iteration $r \in[R]$, we define the following two pseudo-gradients, which are the gradient $\tilde{g}^{r+1}$ with gradient cropping and DP noise addition  and the gradient $g^{r+1}$ without cropping and noise addition,
\end{lemma}
\begin{equation}
	\tilde{g}^{r+1}=\frac{1}{\eta N K} \sum_{i=1}^N-\Delta_i^t=\frac{1}{\eta N K} \sum_{i=1}^N-\tilde{\Delta}_i^t+(1-\beta) g^r,
\end{equation}

\begin{equation}
	g^{r+1}=\beta \frac{1}{S} \sum_i \sum_k \nabla F_i\left(x_i^{r, k}+\delta_i^{r, k} ; \xi_i^{r, k}\right)+(1-\beta) g^r.
\end{equation}
We will come to the following conclusions with Assumption \ref{asp:clip},\\
1. $\mathbb{E}\left[\tilde{g}^{r+1}\right]=g^{r+1}$,\\
2. $\mathbb{E}\left[\|\tilde{g}^{r+1}-g^{r+1}\|\right] \leq \frac{\sigma^2 C^2}{S^2}$.

\section{Appendix A: Basic Assumptions and Notations}
Let $\mathcal{F}^0=\emptyset$ and $\mathcal{F}_i^{r, k}:=\sigma\left(\left\{x_i^{r, j}\right\}_{0 \leq j \leq k} \cup \mathcal{F}^r\right)$ and $\mathcal{F}^{r+1}:=\sigma\left(\cup_i \mathcal{F}_i^{r, K}\right)$ for all $r \geq 0$ where $\sigma(\cdot)$ indicates the $\sigma$-algebra. Let $\mathbb{E}_r[\cdot]:=\overline{\mathbb{E}}\left[\cdot \mid \mathcal{F}^r\right]$ be the expectation, conditioned on the filtration $\mathcal{F}^r$, with respect to the random variables $\left\{\mathcal{S}^r,\left\{\xi_i^{r, k}\right\}_{1 \leq i \leq N, 0 \leq k<K}\right\}$ in the $r$-th iteration. We also use $\mathbb{E}[\cdot]$ to denote the global expectation over all randomness in algorithms. Through out the proofs, we use $\sum_i$ to represent the sum over $i \in\{1, \ldots, N\}$, while $\sum_{i \in \mathcal{S}^r}$ denotes the sum over $i \in \mathcal{S}^r$. Similarly, we use $\sum_k$ to represent the sum of $k \in\{0, \ldots, K-1\}$. For all $r \geq 0$, we define the following auxiliary variables to facilitate proofs:

$$
\begin{aligned}
	\mathcal{E}_r & :=\mathbb{E}\left[\left\|\nabla f\left(\tilde{x}^{r}\right)-g^{r+1}\right\|^2\right] \\
	U_r & :=\frac{1}{N K} \sum_i \sum_k \mathbb{E}\left[\left\|x_i^{r, k}-x^{r}\right\|\right]^2 \\
	\zeta_i^{r, k} & :=\mathbb{E}\left[x_i^{r, k+1}-x_i^{r, k} \mid \mathcal{F}_i^{r, k}\right] \\
	\Xi_r & :=\frac{1}{N} \sum_{i=1}^N \mathbb{E}\left[\left\|\zeta_i^{r, 0}\right\|^2\right] \\
\end{aligned}
$$
$$
\begin{aligned}
\tilde{x}_i^{r, k}=x_i^{r, k+1 / 2}=x_i^{r, k}+\delta_i^{r, k}=x_i^{r, k}+\rho \frac{\tilde{g}^r}{\left\|\tilde{g}^r\right\|}
\end{aligned}
$$
$$
\begin{aligned}
	\tilde{x}^{r+1}=x^{r}+\delta^{r}=x^{r}+\rho \frac{\tilde{g}^r}{\left\|\tilde{g}^r\right\|}
\end{aligned}
$$
 Throughout the appendix, we let $\Delta:=f\left(x^0\right)-f^{\star}, G_0:=\frac{1}{N} \sum_i\left\|\nabla f_i\left(x^0\right)\right\|^2, x^{-1}:=x^0$ and $\mathcal{E}_{-1}:=$ $\mathbb{E}\left[\left\|\nabla f\left(x^0\right)-g^0\right\|^2\right]$. We will use the following foundational lemma for all our algorithms.

\section{DP-FedPGN Algorithm Analyze}
\begin{lemma} \label{lem:Fedavg_grad_err}
	Under Assumption \ref{asp:smooth} , if $\gamma L \leq \frac{1}{24}$, the following holds all $r \geq 0$ :
	
	$$
	\mathbb{E}\left[f\left(\tilde{x}^{r+1}\right)\right] \leq \mathbb{E}\left[f\left(\tilde{x}^{r}\right)\right]-\frac{11 \gamma}{24} \mathbb{E}\left[\left\|\nabla f\left(\tilde{x}^{r}\right)\right\|^2\right]+\frac{13 \gamma}{24} \mathcal{E}_r+\frac{13 \gamma}{24}\frac{\sigma^2 C^2}{S^2}
	$$
\end{lemma}

\begin{proof}
 Since $f$ is $L$-smooth, we have

$$
\begin{aligned}
	f\left(\tilde{x}^{r+1}\right) & \leq f\left(\tilde{x}^{r}\right)+\left\langle\nabla f\left(\tilde{x}^{r}\right), \tilde{x}^{r+1}-\tilde{x}^{r}\right\rangle+\frac{L}{2}\left\|\tilde{x}^{r+1}-\tilde{x}^{r}\right\|^2 \\
	& =f\left(\tilde{x}^{r}\right)-\gamma\left\|\nabla f\left(\tilde{x}^{r}\right)\right\|^2+\gamma\left\langle\nabla f\left(\tilde{x}^{r}\right), \nabla f\left(\tilde{x}^{r}\right)-\tilde{g}^{r+1}\right\rangle+\frac{L \gamma^2}{2}\left\|\tilde{g}^{r+1}\right\|^2 .
\end{aligned}
$$
Since $x^{r+1}=\tilde{x}^{r}-\gamma \tilde{g}^{r+1}$, using Young's inequality, we further have
$$
\begin{aligned}
	&\mathbb{E} f\left(x^{r+1}\right) \\
	& \leq f\left(\tilde{x}^{r}\right)-\frac{\gamma}{2}\left\|\nabla f\left(\tilde{x}^{r}\right)\right\|^2+\frac{\gamma}{2}\left\|\nabla f\left(\tilde{x}^{r}\right)-\tilde{g}^{r+1}\right\|^2+L \gamma^2\left(\left\|\nabla f\left(\tilde{x}^{r}\right)\right\|^2+\left\|\nabla f\left(\tilde{x}^{r}\right)-\tilde{g}^{r+1}\right\|^2\right) \\
	& \leq f\left(\tilde{x}^{r}\right)-\frac{11 \gamma}{24}\left\|\nabla f\left(\tilde{x}^{r}\right)\right\|^2+\frac{13 \gamma}{24}\left\|\nabla f\left(\tilde{x}^{r}\right)-\tilde{g}^{r+1}\right\|^2\\
	& \leq f\left(\tilde{x}^{r}\right)-\frac{11 \gamma}{24}\left\|\nabla f\left(\tilde{x}^{r}\right)\right\|^2+\frac{13 \gamma}{24}\left\|\nabla f\left(\tilde{x}^{r}\right)-g^{r+1}\right\|^2+\frac{13 \gamma}{24}\frac{\sigma^2 C^2}{S^2}
\end{aligned}
$$
where the last inequality is due to $\gamma L \leq \frac{1}{24}$ and Lemma \ref{lem:dp-noise}, Taking the global expectation completes the proof.
\end{proof}

\begin{lemma}\label{lem:Fedavg_client_drift}
	If $\gamma L \leq \frac{\beta}{6}$, the following holds for $r \geq 1$ :
	
	$$
	\mathcal{E}_r \leq\left(1-\frac{8 \beta}{9}\right) \mathcal{E}_{r-1}+\frac{4 \gamma^2 L^2}{\beta} \mathbb{E}\left[\left\|\nabla f\left(\tilde{x}^{r-1}\right)\right\|^2\right]+\frac{2 \beta^2 \sigma_l^2}{S K}+4 \beta L^2 U_r
	$$
	Additionally, it holds for $r=0$ that
	$$
	\mathcal{E}_0 \leq(1-\beta) \mathcal{E}_{-1}+\frac{2 \beta^2 \sigma_l^2}{S K}+4 \beta L^2 U_0
	$$
\end{lemma}

\begin{proof}
For $r>1$,
$$
\begin{aligned}
	\mathcal{E}_r= & \mathbb{E}\left[\left\|\nabla f\left(\tilde{x}^{r}\right)-g^{r+1}\right\|^2\right] \\
	= & \mathbb{E}\left[\left\|(1-\beta)\left(\nabla f\left(\tilde{x}^{r}\right)-g^r\right)+\beta\left(\nabla f\left(\tilde{x}^{r}\right)-\frac{1}{SK} \sum_i \sum_k \nabla F\left(\tilde{x}_i^{r, k} ; \xi_i^{r, k}\right)\right)\right\|^2\right] \\
	= & \mathbb{E}\left[\left\|(1-\beta)\left(\nabla f\left(\tilde{x}^{r}\right)-g^r\right)\right\|^2\right]+\beta^2 \mathbb{E}\left[\left\|\nabla f\left(\tilde{x}^{r}\right)-\frac{1}{S K} \sum_{i, k} \nabla F\left(\tilde{x}_i^{r, k} ; \xi_i^{r, k}\right)\right\|^2\right] \\
	& +2 \beta \mathbb{E}\left[\left\langle(1-\beta)\left(\nabla f\left(\tilde{x}^{r}\right)-g^r\right), \nabla f\left(\tilde{x}^{r}\right)-\frac{1}{N K} \sum_{i, k} \nabla f\left(\tilde{x}_i^{r, k}\right)\right\rangle\right] .
\end{aligned}
$$
Note that $\left\{\nabla F\left(\tilde{x}_i^{r, k} ; \xi_i^{r, k}\right)\right\}_{0 \leq k<K}$ are sequentially correlated. Applying the AM-GM inequality and Lemma  \ref{lem:bias-var}, we have
$$
\mathcal{E}_r \leq\left(1+\frac{\beta}{2}\right) \mathbb{E}\left[\left\|(1-\beta)\left(\nabla f\left(\tilde{x}^{r}\right)-g^r\right)\right\|^2\right]+2 \beta L^2 U_r+2 \beta^2\left(\frac{\sigma_l^2}{S K}+L^2 U_r\right)
$$
Using the AM-GM inequality again and Assumption \ref{asp:smooth}, we have
$$
\begin{aligned}
	\mathcal{E}_r & \leq(1-\beta)^2\left(1+\frac{\beta}{2}\right)\left[\left(1+\frac{\beta}{2}\right) \mathcal{E}_{r-1}+\left(1+\frac{2}{\beta}\right) L^2 \mathbb{E}\left[\left\|\tilde{x}^{r}-\tilde{x}^{r-1}\right\|^2\right]\right]+\frac{2 \beta^2 \sigma_l^2}{S K}+4 \beta L^2 U_r \\
	& \leq(1-\beta) \mathcal{E}_{r-1}+\frac{2}{\beta} L^2 \mathbb{E}\left[\left\|\tilde{x}^{r}-\tilde{x}^{r-1}\right\|^2\right]+\frac{2 \beta^2 \sigma_l^2}{S K}+4 \beta L^2 U_r \\
	& \leq\left(1-\frac{8 \beta}{9}\right) \mathcal{E}_{r-1}+4 \frac{\gamma^2 L^2}{\beta} \mathbb{E}\left[\left\|\nabla f\left(\tilde{x}^{r-1}\right)\right\|^2\right]+\frac{2 \beta^2 \sigma_l^2}{S K}+4 \beta L^2 U_r
\end{aligned}
$$

where we plug in $\left\|\tilde{x}^{r}-\tilde{x}^{r-1}\right\|^2 \leq 2 \gamma^2\left(\left\|\nabla f\left(\tilde{x}^{r-1}\right)\right\|^2+\left\|g^r-\nabla f\left(\tilde{x}^{r-1}\right)\right\|^2\right)$ and use $\gamma L \leq \frac{\beta}{6}$ in the last inequality. Similarly for $r=0$,

$$
\begin{aligned}
	\mathcal{E}_0 & \leq\left(1+\frac{\beta}{2}\right) \mathbb{E}\left[\left\|(1-\beta)\left(\nabla f\left(\tilde{x}^0\right)-g^0\right)\right\|^2\right]+2 \beta L^2 U_0+2 \beta^2\left(\frac{\sigma_l^2}{S K}+L^2 U_0\right) \\
	& \leq(1-\beta) \mathcal{E}_{-1}+\frac{2 \beta^2 \sigma_l^2}{S K}+4 \beta L^2 U_0
\end{aligned}
$$
\end{proof}

\begin{lemma}
	If $\eta L K \leq \frac{1}{\beta}$, the following holds for $r \geq 0$ :
	
	$$
	U_r \leq 2 e K^2 \Xi_r+K \eta^2 \beta^2 \sigma_l^2\left(1+2 K^3 L^2 \eta^2 \beta^2\right)
	$$
\end{lemma}
\begin{proof}
Recall that $\zeta_i^{r, k}:=\mathbb{E}\left[x_i^{r, k+1}-x_i^{r, k} \mid \mathcal{F}_i^{r, k}\right]=-\eta\left((1-\beta) g^r+\beta \nabla f_i\left(\tilde{x}_i^{r, k}\right)\right)$. Then we have
$$
\begin{aligned}
	\mathbb{E}\left[\left\|\zeta_i^{r, j}-\zeta_i^{r, j-1}\right\|^2\right] & \leq \eta^2 L^2 \beta^2 \mathbb{E}\left[\left\|x_i^{r, j}-x_i^{r, j-1}\right\|^2\right] \\
	& \leq \eta^2 L^2 \beta^2\left(\eta^2 \beta^2 \sigma_l^2+\mathbb{E}\left[\left\|\zeta_i^{r, j-1}\right\|^2\right)\right.
\end{aligned}
$$
For any $1 \leq j \leq k-1 \leq K-2$, using $\eta L \leq \frac{1}{\beta K} \leq \frac{1}{\beta(k+1)}$, we have
$$
\begin{aligned}
	\mathbb{E}\left[\left\|\zeta_i^{r, j}\right\|^2\right] & \leq\left(1+\frac{1}{k}\right) \mathbb{E}\left[\left\|\zeta_i^{r, j-1}\right\|^2\right]+(1+k) \mathbb{E}\left[\left\|\zeta_i^{r, j}-\zeta_i^{r, j-1}\right\|^2\right] \\
	& \leq\left(1+\frac{2}{k}\right) \mathbb{E}\left[\left\|\zeta_i^{r, j-1}\right\|^2\right]+(k+1) L^2 \eta^4 \beta^4 \sigma_l^2 \\
	& \leq e^2 \mathbb{E}\left[\left\|\zeta_i^{r, 0}\right\|^2\right]+4 k^2 L^2 \eta^4 \beta^4 \sigma_l^2
\end{aligned}
$$
where the last inequality is by unrolling the recursive bound and using $\left(1+\frac{2}{k}\right)^k \leq e^2$. By Lemma \ref{lem:bias-var} , it holds that for $k \geq 2$,
$$
\begin{aligned}
	\mathbb{E}\left[\left\|x_i^{r, k}-\tilde{x}^{r}\right\|^2\right] & \leq 2 \mathbb{E}\left[\left\|\sum_{j=0}^{k-1} \zeta_i^{r, j}\right\|^2\right]+2 k \eta^2 \beta^2 \sigma_l^2 \\
	& \leq 2 k \sum_{j=0}^{k-1} \mathbb{E}\left[\left\|\zeta_i^{r, k}\right\|^2\right]+2 k \eta^2 \beta^2 \sigma_l^2 \\
	& \leq 2 e^2 k^2 \mathbb{E}\left[\left\|\zeta_i^{r, 0}\right\|^2\right]+2 k \eta^2 \beta^2 \sigma_l^2\left(1+4 k^3 L^2 \eta^2 \beta^2\right)
\end{aligned}
$$
This is also valid for $k=0,1$. Summing up over $i$ and $k$ finishes the proof.
\end{proof}

\begin{lemma}\label{lem:Fedavg_grad_norm}
	If $288 e(\eta K L)^2\left((1-\beta)^2+e(\beta \gamma L R)^2\right) \leq 1$, then it holds for $r \geq 0$ that
	
	$$
	\sum_{r=0}^{R-1} \Xi_r \leq \frac{1}{72 e K^2 L^2} \sum_{r=-1}^{R-2}\left(\mathcal{E}_r+\mathbb{E}\left[\left\|\nabla f\left(\tilde{x}^{r}\right)\right\|^2\right]\right)+2 \eta^2 \beta^2 e R G_0
	$$
\end{lemma}
\begin{proof}
 Note that $\zeta_i^{r, 0}=-\eta\left((1-\beta) g^r+\beta \nabla f_i\left(\tilde{x}^{r}\right)\right)$,
$$
\frac{1}{N} \sum_{i=1}^N\left\|\zeta_i^{r, 0}\right\|^2 \leq 2 \eta^2\left((1-\beta)^2\left\|g^r\right\|^2+\beta^2 \frac{1}{N} \sum_{i=1}^N\left\|\nabla f_i\left(x^{r}+\delta_{i}^{r,k}\right)\right\|^2\right)
$$
Using Young's inequality, we have for any $q>0$ that
$$
\begin{aligned}
	\mathbb{E}\left[\left\|\nabla f_i\left(\tilde{x}^{r}\right)\right\|^2\right] & \leq(1+q) \mathbb{E}\left[\left\|\nabla f_i\left(\tilde{x}^{r-1}\right)\right\|^2\right]+\left(1+q^{-1}\right) L^2 \mathbb{E}\left[\left\|\tilde{x}^{r}-\tilde{x}^{r-1}\right\|^2\right] \\
	& \leq(1+q) \mathbb{E}\left[\left\|\nabla f_i\left(\tilde{x}^{r-1}\right)\right\|^2\right]+2\left(1+q^{-1}\right) \gamma^2 L^2\left(\mathcal{E}_{r-1}+\mathbb{E}\left[\left\|\nabla f\left(\tilde{x}^{r-1}\right)\right\|^2\right]\right) \\
	& \leq(1+q)^r \mathbb{E}\left[\left\|\nabla f_i\left(\tilde{x}^0\right)\right\|^2\right]+\frac{2}{q} \gamma^2 L^2 \sum_{j=0}^{r-1}\left(\mathcal{E}_j+\mathbb{E}\left[\left\|\nabla f\left(\tilde{x}^j\right)\right\|^2\right)(1+q)^{r-j}\right.
\end{aligned}
$$
Take $q=\frac{1}{r}$ and we have

\begin{equation}\label{eqn:vninbvsdvds}
\mathbb{E}\left[\left\|\nabla f_i\left(\tilde{x}^{r}\right)\right\|^2\right] \leq e \mathbb{E}\left[\left\|\nabla f_i\left(\tilde{x}^0\right)\right\|^2\right]+2 e(r+1) \gamma^2 L^2 \sum_{j=0}^{r-1}\left(\mathcal{E}_j+\mathbb{E}\left[\left\|\nabla f\left(\tilde{x}^j\right)\right\|^2\right)\right.
\end{equation}
Note that this inequality is valid for $r=0$. Therefore, using \eqref{eqn:vninbvsdvds}, we have
$$
\begin{aligned}
	\sum_{r=0}^{R-1} \Xi_r \leq & \sum_{r=0}^{R-1} 2 \eta^2 \mathbb{E}\left[(1-\beta)^2\left\|g^r\right\|^2+\beta^2 \frac{1}{N} \sum_{i=1}^N\left\|\nabla f_i\left(\tilde{x}^{r}\right)\right\|^2\right] \\
	\leq & \sum_{r=0}^{R-1} 2 \eta^2\left(2(1-\beta)^2\left(\mathcal{E}_{r-1}+\mathbb{E}\left[\left\|\nabla f\left(\tilde{x}^{r-1}\right)\right\|^2\right]\right)+\beta^2 \frac{1}{N} \sum_{i=1}^N \mathbb{E}\left[\left\|\nabla f_i\left(\tilde{x}^{r}\right)\right\|^2\right]\right) \\
	\leq & \sum_{r=0}^{R-1} 4 \eta^2(1-\beta)^2\left(\mathcal{E}_{r-1}+\mathbb{E}\left[\left\|\nabla f\left(\tilde{x}^{r-1}\right)\right\|^2\right]\right) \\
	& +2 \eta^2 \beta^2 \sum_{r=0}^{R-1}\left(\frac{e}{N} \sum_{i=1}^N \mathbb{E}\left[\left\|\nabla f_i\left(\tilde{x}^0\right)\right\|^2\right]+2 e(r+1)(\gamma L)^2 \sum_{j=0}^{r-1}\left(\mathcal{E}_j+\mathbb{E}\left[\left\|\nabla f\left(\tilde{x}^j\right)\right\|^2\right]\right)\right) \\
	\leq & 4 \eta^2(1-\beta)^2 \sum_{r=0}^{R-1}\left(\mathcal{E}_{r-1}+\mathbb{E}\left[\left\|\nabla f\left(\tilde{x}^{r-1}\right)\right\|^2\right]\right) \\
	& +2 \eta^2 \beta^2\left(e R G_0+2 e(\gamma L R)^2 \sum_{r=0}^{R-2}\left(\mathcal{E}_r+\mathbb{E}\left[\left\|\nabla f\left(\tilde{x}^{r}\right)\right\|^2\right]\right)\right)
\end{aligned}
$$
Rearranging the equation and applying the upper bound of $\eta$ completes the proof.
\end{proof}
\begin{theorem}[Convergence for non-convex functions]	
	Under Assumption \ref{asp:smooth} and \ref{asp:sgd_var} and \ref{asp:clip} , if we take $g^0=0$,
	$$
	\begin{aligned}
		& \beta=\min \left\{, \sqrt{\frac{S K L \Delta}{\sigma_l^2 R}}\right\} \text { for any constant } c \in(0,1], \quad \gamma=\min \left\{\frac{1}{24 L}, \frac{\beta}{6 L}\right\}, \\
		& \eta K L \lesssim \min \left\{1, \frac{1}{\beta \gamma L R},\left(\frac{L \Delta}{G_0 \beta^3 R}\right)^{1 / 2}, \frac{1}{(\beta N)^{1 / 2}}, \frac{1}{\left(\beta^3 N K\right)^{1 / 4}}\right\}
	\end{aligned}
	$$
	then DP-FedPGN converges as
	
	$$
	\frac{1}{R} \sum_{r=0}^{R-1} \mathbb{E}\left[\left\|\nabla f\left(\tilde{x}^{r}\right)\right\|^2\right] \lesssim \sqrt{\frac{L \Delta \sigma_l^2}{S K R}}+\frac{L \Delta}{R} .
	$$

	Here $G_0:=\frac{1}{N} \sum_{i=1}^N\left\|\nabla f_i\left(\tilde{x}^0\right)\right\|^2$.
\end{theorem}
\begin{proof}
 Combining Lemma \ref{lem:Fedavg_grad_err} and \ref{lem:Fedavg_client_drift}, we have

$$
\begin{aligned}
	\mathcal{E}_r \leq & \left(1-\frac{8 \beta}{9}\right) \mathcal{E}_{r-1}+4 \frac{(\gamma L)^2}{\beta} \mathbb{E}\left[\left\|\nabla f\left(\tilde{x}^{r-1}\right)\right\|^2\right]+\frac{2 \beta^2 \sigma_l^2}{S K} \\
	& +4 \beta L^2\left(2 e K^2 \Xi_r+K \eta^2 \beta^2 \sigma_l^2\left(1+2 K^3 L^2 \eta^2 \beta^2\right)\right.
\end{aligned}
$$
and
$$
\mathcal{E}_0 \leq(1-\beta) \mathcal{E}_{-1}+\frac{2 \beta^2 \sigma_l^2}{S K}+4 \beta L^2\left(2 e K^2 \Xi_0+K \eta^2 \beta^2 \sigma_l^2\left(1+2 K^3 L^2 \eta^2 \beta^2\right)\right) .
$$
Summing over $r$ from 0 to $R-1$ and applying Lemma \ref{lem:Fedavg_grad_norm},
$$
\begin{aligned}
	\sum_{r=0}^{R-1} \mathcal{E}_r \leq & \left(1-\frac{8 \beta}{9}\right) \sum_{r=-1}^{R-2} \mathcal{E}_r+4 \frac{(\gamma L)^2}{\beta} \sum_{r=0}^{R-2} \mathbb{E}\left[\left\|\nabla f\left(\tilde{x}^{r}\right)\right\|^2\right]+2 \frac{\beta^2 \sigma_l^2}{S K} R \\
	& +4 \beta L^2\left(2 e K^2 \sum_{r=0}^{R-1} \Xi_r+R K \eta^2 \beta^2 \sigma_l^2\left(1+2 K^3 L^2 \eta^2 \beta^2\right)\right) \\
	\leq & \left(1-\frac{7 \beta}{9}\right) \sum_{r=-1}^{R-2} \mathcal{E}_r+\left(4 \frac{(\gamma L)^2}{\beta}+\frac{\beta}{9}\right) \sum_{r=-1}^{R-2} \mathbb{E}\left[\left\|\nabla f\left(\tilde{x}^{r}\right)\right\|^2\right]+16 \beta^3(e \eta K L)^2 R G_0 \\
	& +\frac{2 \beta^2 \sigma_l^2}{S K} R+4 \beta^3(\eta K L)^2\left(\frac{1}{K}+2(\eta K L \beta)^2\right) \sigma_l^2 R \\
	\leq & \left(1-\frac{7 \beta}{9}\right) \sum_{r=-1}^{R-2} \mathcal{E}_r+\frac{2 \beta}{9} \sum_{r=-1}^{R-2} \mathbb{E}\left[\left\|\nabla f\left(\tilde{x}^{r}\right)\right\|^2\right]+16 \beta^3(e \eta K L)^2 R G_0+\frac{4 \beta^2 \sigma_l^2}{S K} R
\end{aligned}
$$
Here in the last inequality we apply
$$
4 \beta(\eta K L)^2\left(\frac{1}{K}+2(\eta K L \beta)^2\right) \leq \frac{2}{N K} \quad \text { and } \quad \gamma L \leq \frac{\beta}{6} .
$$
Therefore,
$$
\sum_{r=0}^{R-1} \mathcal{E}_r \leq \frac{9}{7 \beta} \mathcal{E}_{-1}+\frac{2}{7} \mathbb{E}\left[\sum_{r=-1}^{R-2}\left\|\nabla f\left(\tilde{x}^{r}\right)\right\|^2\right]+\frac{144}{7}(e \beta \eta K L)^2 G_0 R+\frac{36 \beta \sigma_l^2}{7 S K} R .
$$
Combine this inequality with Lemma \ref{lem:Fedavg_grad_err} and we get
$$
\frac{1}{\gamma} \mathbb{E}\left[f\left(\tilde{x}^{r}\right)-f\left(\tilde{x}^0\right)\right] \leq-\frac{1}{7} \sum_{r=0}^{R-1} \mathbb{E}\left[\left\|\nabla f\left(\tilde{x}^{r}\right)\right\|^2\right]+\frac{39}{56 \beta} \mathcal{E}_{-1}+\frac{78}{7}(e \beta \eta K L)^2 G_0 R+\frac{39 \beta \sigma_l^2}{14 S K} R +\frac{13}{24}\frac{\sigma^2 C^2}{S^2}R
$$
Finally, noticing that $g^0=0$ implies $\mathcal{E}_{-1} \leq 2 L\left(f\left(\tilde{x}^0\right)-f^*\right)=2 L \Delta$, we obtain
$$
\begin{aligned}
	\frac{1}{R} \sum_{r=0}^{R-1} \mathbb{E}\left[\left\|\nabla f\left(\tilde{x}^{r}\right)\right\|^2\right] & \lesssim \frac{L \Delta}{\gamma L R}+\frac{\mathcal{E}_{-1}}{\beta R}+(\beta \eta K L)^2 G_0+\frac{\beta \sigma_l^2}{S K}+\frac{\sigma^2 C^2}{S^2} \\
	& \lesssim \frac{L \Delta}{R}+\frac{L \Delta}{\beta R}+\frac{\beta \sigma_l^2}{S K}+(\beta \eta K L)^2 G_0 +\frac{\sigma^2 C^2}{S^2}\\
	& \lesssim \frac{L \Delta}{R}+\sqrt{\frac{L \Delta \sigma_l^2}{S K R}}+\frac{\sigma^2 C^2}{S^2}
\end{aligned}
$$
\end{proof}




\end{document}